\documentclass{article}

\usepackage{microtype}
\usepackage{graphicx}
\usepackage{subfigure}
\usepackage{booktabs} 

\usepackage{hyperref}

\usepackage[accepted]{icml2023}

\usepackage{amsmath}
\usepackage{amssymb}
\usepackage{mathtools}
\usepackage{amsthm}

\usepackage[capitalize,noabbrev]{cleveref}

\theoremstyle{plain}
\newtheorem{theorem}{Theorem}[section]
\newtheorem{proposition}[theorem]{Proposition}

\theoremstyle{definition}
\newtheorem{definition}[theorem]{Definition}

\theoremstyle{remark}

\usepackage{wrapfig}        
\usepackage{multirow,multicol}
\usepackage{amsmath,amssymb,amsfonts,graphicx,amsthm}
\usepackage{color,soul,xspace}
\usepackage{booktabs}       
\usepackage{physics}        
\usepackage{colortbl}  

\newcommand{\eat}[1]{}

\newcommand{\lb}{\left(}
\newcommand{\rb}{\right)}

\definecolor{Gray}{gray}{0.95}
\definecolor{Cyan}{rgb}{0.88,1,1}
\definecolor{Blu}{RGB}{68,115,196} 

\newcolumntype{a}{>{\columncolor{Gray}}c}

\usepackage{dblfloatfix}
\usepackage{rotating}
\usepackage{adjustbox}

\usepackage[textsize=tiny]{todonotes}

\icmltitlerunning{Feature Expansion for Graph Neural Networks}

\begin{document}

\twocolumn[
\icmltitle{Feature Expansion for Graph Neural Networks}



\icmlsetsymbol{equal}{*}

\begin{icmlauthorlist}
\icmlauthor{Jiaqi Sun}{sigs}
\icmlauthor{Lin Zhang}{idea}
\icmlauthor{Guangyi Chen}{cmu,mbz}
\icmlauthor{Peng Xu}{cuhk}
\icmlauthor{Kun Zhang}{cmu,mbz}
\icmlauthor{Yujiu Yang}{sigs}
\end{icmlauthorlist}

\icmlaffiliation{sigs}{Shenzhen International Graduate School, Tsinghua University, China}
\icmlaffiliation{cmu}{Carnegie Mellon University, USA}
\icmlaffiliation{mbz}{Mohamed bin Zayed University of Artificial Intelligence, UAE}
\icmlaffiliation{idea}{International Digital Economy Academy, China}
\icmlaffiliation{cuhk}{Chinese University of Hong Kong, China}

\icmlcorrespondingauthor{Yujiu Yang}{yang.yujiu@sz.tsinghua.edu.cn}
\icmlcorrespondingauthor{Lin Zhang}{linzhang0529@gmail.com}

\icmlkeywords{Machine Learning, ICML}

\vskip 0.3in
]



\printAffiliationsAndNotice{}  

\begin{abstract}
Graph neural networks aim to learn representations for graph-structured data and show impressive performance, particularly in node classification.
Recently, many methods have studied the representations of GNNs from the perspective of optimization goals and spectral graph theory.
However, the feature space that dominates representation learning has not been systematically studied in graph neural networks.
In this paper, we propose to fill this gap by analyzing the feature space of both spatial and spectral models.
We decompose graph neural networks into determined feature spaces and trainable weights, providing the convenience of studying the feature space explicitly using matrix space analysis.
In particular, we theoretically find that the feature space tends to be linearly correlated due to repeated aggregations.
In this case, the feature space is bounded by the poor representation of shared weights or the limited dimensionality of node attributes in existing models, leading to poor performance.
Motivated by these findings, we propose 1) feature subspaces flattening and 2) structural principal components to expand the feature space.
%
Extensive experiments verify the effectiveness of our proposed more comprehensive feature space, with comparable inference time to the baseline, and demonstrate its efficient convergence capability.

\end{abstract}

\section{Introduction}
Graph Neural Networks~(GNNs) have shown great potential in learning representations of graph-structured data, such as social networks, transportation networks, protein interaction networks, etc.~\cite{fan2019graph,wu2020comprehensive,khoshraftar2022survey}.
In this paper, we focus on node representation learning, which is one of the most important tasks in this line of research, where the key is to represent nodes in an informative and structure-aware way.

There are two different types of graph neural networks.
One is spatial, which aggregates information from neighboring nodes and updates the representation of the central node.~\cite{velickovic2018gat,xu2018powerful,huang2020combining}. 
%
The spectral type, on the other hand, treats the graph structure matrix, such as the Laplacian matrix, as a transformation for the nodes' attributes~(signals) in the spectral domain~\citep{defferrard2016convolutional,chien2021gprgnn,he2021bernnet,he2022convolutional}. The aim is to develop flexible functions for the graph structure so that the signals of the nodes can fit the labels appropriately.

Recently, several perspectives for analyzing GNN representations have emerged, such as general optimization functions, denoising frameworks, and spectral graph theory~\cite{zhu2021interpreting,unifiedgnn_ma,balcilar2021bridge}.
However, as a determinant of representation learning, feature spaces have not been systematically studied for graph neural networks.
In general representation learning, performance depends heavily on the construction of feature spaces with accessible data~\cite{bengio2013representation}.

In this paper, we propose to fill this gap and investigate the feature space for both spatial and spectral GNNs.
Specifically, for theoretical investigations, we first abstract a linear approximation of the GNNs following the studies~\cite{wu2019sgc,xu2018powerful,wang2022powerful}.
Then, we decompose the GNN components with and without parameters in the linear approximation, where the latter is considered as a feature space built by node attributes and graph structure (e.g., adjacency or Laplacian matrices), and the former denotes the learnable parameters to reweight the features.

Taking advantage of the convenience of decomposition, we examine the feature space of current models.
Motivated by that GNNs are expected to fit arbitrary objective, a more comprehensive feature space reflects better representability without any assumption about the data distribution.
However, we find theoretically that 
the feature subspaces of current GNNs are bounded by the weight-sharing mechanism and the limited dimensionality of the node attributes.
In order to alleviate the above restrictions and expand the feature space, we proposed \textit{1) feature subspace flattening} and \textit{2) structural principal components}, respectively.
Specifically, the former reweights all feature subspaces independently to obtain a fully expressed representation.
The latter adds the principal components of graph structure matrices as a "complement" to the original feature space.
%
%
%
%
It is emphasized that our proposal makes no assumptions about the graph or the target, which enjoys good generality.
We perform extensive experiments on both homophilic and heterophilic datasets to demonstrate the superiority of the proposal.

Our contributions are listed below:
\begin{itemize}
    \item Starting from representation learning, we provide the first study of the feature space formed in graph-structured data. Based on this view, we study typical spatial and spectral GNNs and identify two problems of existing GNNs caused by bounded feature spaces.
    \item We then propose two modifications: 1) feature subspace flattening and 2) structural principal components to expand the whole feature space.
    \item Extensive experiments are performed on homophilic and heterophilic datasets, and our proposal achieves significant improvements, e.g. an average accuracy increase of 32\% on heterophilic graphs.

\end{itemize}



\section{Preliminaries}
\label{sec:preliminaries}

\begin{table*}[t]
    \centering
    \caption{Feature space and parameters for GNN models (better viewed in color)}
    \begin{adjustbox}{width=\textwidth}
    \renewcommand{\arraystretch}{1.2}
    \begin{tabular}{cll}
    \toprule
         &  Original formula$^{*}$ &  Linear approximation formulations \\
        \midrule
        GCN & \multirow{2}{*}{$H^{(k+1)}=\sigma\lb\hat{A}H^{(k)}W^{(k)}\rb$} 
        & \multirow{2}{*}{$H^{(K)}=\textcolor{orange}{\hat{A}^KX}\textcolor{Blu}{\prod_{i=0}^{K-1} W^{(i)}}$}\\
        {\cite{kipf2016gcn}}&&\\
        GIN & \multirow{2}{*}{$H^{(k+1)}=\sigma\lb (\epsilon^{(k)}I+\hat{A}) H^{(k)} W_0^{(k)}\rb W_1^{(k)}$} 
        & \multirow{2}{*}{$H^{(K)}=\sum_{t=0}^{K} \textcolor{orange}{\hat{A}^kX} \textcolor{Blu}{\sum_{\{q_0,\cdots,q_{K-t-1}\}\subseteq\{\epsilon^{(0)},\cdots,\epsilon^{(K-1)}\}}\prod_i q_i\cdot\prod_{j=0}^{K-1}W_0^{(j)}W_1^{(j)}} $}\\
        {\citep{xu2018powerful}}&&\\
        GCNII & \multirow{2}{*}{$ H^{(l+1)}=\sigma\lb \lb (1-\alpha^{(l)})\hat{A}H^{(l)}+\alpha^{(l)}H^{(0)} \rb \lb (1-\beta^{(l)})I+\beta^{(l)}W^{(l)} \rb \rb$ }
        & \multirow{2}{*}{$H^{(K)}=\sum_{l=0}^{K-1} \textcolor{orange}{\hat{A}^lX}\textcolor{Blu}{\prod_{i=L-l}^{L-1}(1-\alpha^{(i)})\alpha^{(L-l-1)}\prod_{j=L-l-1}^{L-1}W^{(j)}\}}+ \textcolor{orange}{\hat{A}^K}\textcolor{Blu}{\prod_{h=0}^{K-1}(1-\alpha^{(h)})W^{(h)}}$}\\
        {\citep{chen2020simple}}&&\\
        ARMA & \multirow{2}{*}{$H^{(K)}=\sigma(\tilde{L}H^{(K-1)}W_1+XW_2)$} & \multirow{2}{*}{$H^{(K)}=\sum_{t=0}^K \textcolor{orange}{\tilde{L}X} \textcolor{Blu}{W_2^{t}W_1^{K-t}}$} \\
        {\citep{bianchi2021arma}}&&\\
        APPNP & \multirow{2}{*}{$H^{(k+1)}=(1-\alpha)\hat{A}H^{(l)}+\alpha H^{(0)};H^{(0)}=\sigma(XW_1)W_2$} 
        & \multirow{2}{*}{$H^{(K)}=\sum_{t=0}^{K} \textcolor{orange}{(1-\alpha)^t\hat{A}^lH^{(0)}+\sum_{i=0}^{t-1}\alpha(1-\alpha)^i\hat{A}^i H^{(0)}} \textcolor{Blu}{W_1W_2} $}\\
        {\citep{klicpera2019appnp}}&&\\
        ChebyNet$^{**}$ & \multirow{2}{*}{$H= \sum_{k=0}^KP_k(\hat{L})XW^{(k)}$} 
        & \multirow{2}{*}{$H^{(K)}=\sum_{t=0}^{K} \textcolor{orange}{P_t(\hat{L})X} \textcolor{Blu}{W^{(t)}} $}\\
        {\citep{defferrard2016convolutional}}&&\\
        GPRGNN & \multirow{2}{*}{$H=\sum_{k=0}^K\gamma^{(k)}\hat{L}^k\sigma(XW_1)W_2$ }
        & \multirow{2}{*}{$H^{(K)}=\sum_{t=0}^{K} \textcolor{orange}{\hat{L}^tX}\textcolor{Blu}{\gamma^{(t)}W_1W_2}$}\\
        {\citep{chien2021gprgnn}}&&\\
        BernNet & \multirow{2}{*}{$H=\sum_{k=0}^K\frac{1}{2^K}\binom{K}{k}\gamma^{(k)}(2I-\hat{L})^{K-k}\hat{L}^k\sigma(XW_1)W_2$} 
        & \multirow{2}{*}{$H^{(K)}=\sum_{t=0}^{K} \textcolor{orange}{\sum_{j=0}^{t} \frac{1}{2^j}  \tbinom{K}{j}  \hat{L}^{t}}\textcolor{Blu}{\sum_{j=0}^{t} \gamma^{(j)} W_1W_2}$}\\
        {\citep{he2021bernnet}}&&\\
        \bottomrule
        \multicolumn{3}{l}{$^*$ Without specification, $H^{(0)}=X$.}\\
        \multicolumn{3}{l}{$^{**}$ $T_k(x)$ denotes Chebyshev polynomial $P_0(x)=1,P_1(x)=x,P_k(x)=2xP_{k-1}-P_{k-2}$.}\\
    \end{tabular}
    \end{adjustbox}
    \label{tab:gnn_elements_big}
\end{table*}
In this paper, we focus on the undirected graph $\mathcal{G}=(\mathcal{V}, \mathcal{E})$, along with its node attributes of $\mathcal{V}$ as $X\in\mathbb{R}^{n\times d}$ and adjacency matrix $A\in\mathbb{R}^{n\times n}$ to present $\mathcal{E}$. 
GNNs take the input of the node attributes and the adjacency matrix, and output the hidden node representations, as $H=\mathtt{GNN}(X, A) \in\mathbb{R}^{n\times d}$.
By default, we employ the cross-entropy loss function in the node classification task to minimize the difference between node label $Y$ and the obtained representation as $\mathcal{L}(H,Y) = -\sum_i Y_i \log \mathrm{softmax}(H_i)$.

\textbf{Spatial GNNs~(with non-parametric aggregation)} mostly fall into the message-passing paradigm. 
For any given node, it essentially aggregates features from its neighbors and updates the aggregated feature,
\begin{equation}
\begin{aligned}
    H_i^{(k+1)}=\sigma\lb f_{u}\lb H_i^{(k)},f_{a}\lb{\hat{A}_{ij},H_j^{(k)};j\in \mathcal{N}_i}\rb\rb\rb,
    \label{eq:spatial_gnn}
\end{aligned}
\end{equation}
where $\sigma\lb\cdot\rb$ is a non-linear activation function, $H^{(k)}$ indicates the hidden representation in $k$-th layer, $f_{a}$ and $f_{u}$ are the aggregation and updating functions~\citep{balcilar2021bridge}, 
$\hat{A}=(D+I)^{-1/2}(A+I)(D+I)^{-1/2}$ is the re-normalized adjacency matrix using the degree matrix $D$, and $\mathcal{N}_i$ denotes the $1$-hop neighbors.

Here, we provide two examples to specify this general expression. 
One is the vanilla GCN\footnote{GCN~\cite{kipf2016gcn} is also an instance of spectral GNNs, and here we categorize it into spatial type due to its critical role of bridging spectral GNNs to spatial interpretations.}
~\cite{kipf2016gcn} that adopts the mean-aggregation and the average-update, whose formulation is:
\begin{align}
    H^{(k+1)} = \sigma\lb\hat{A}H^{(k)}W^{(k)} \rb.
    \label{eq:gcn}
\end{align} 

The second example shows a different update scheme with skip-connection~\citep{xu2018powerful, li2019deepgcns, chen2020simple}, which is defined as follows,
\begin{align}
    H^{(k+1)} = \sigma\lb \alpha^{(k)} H^{(0)}W^{(k)}_0 + \hat{A}H^{(k)}W^{(k)}_1 \rb,
    \label{eq:gin}
\end{align}
where $\alpha^{(k)}$ controls the weight of each layer's skip-connection, $W_0^{(k)}, W_1^{(k)}$ are the transformation weights for the initial layer and the previous one, respectively. 

\textbf{Spectral GNNs (polynomial-based)} 
originally employ the Graph Fourier transforms to get filters~\citep{chung1997spectral}, such as using the eigendecomposition of the Laplacian matrix: $\hat{L}=I-\hat{A}=U\Lambda U^T$. In recent years, methods of this type have focused more on approximating arbitrary global filters using polynomials~\citep{wang2022powerful,zhu2020simple, he2021bernnet}, which has shown superior performance and is written as
\begin{align}
    H = \sum_{k=0}^K \gamma^{(k)}P_k(\hat{L})\sigma(XW_1)W_2,
    \label{eq:spectral_gnn}
\end{align}
where $P_k(\cdot)$ donates a polynomial's $k$-order term; $\gamma^{(k)}$ is the adaptive coefficients and $W_1,W_2$ are learnable parameters. 

In this paper, we focus on the typical instances of the two types of graph neural networks, as indicated in the parentheses above.
For spatial models, we focus on those with non-parametric aggregation, excluding learnable aggregation such as \citep{velickovic2018gat} in the analytical parts.
Considering the spectral type, we focus on polynomial-based models, and other spectral filters are not included in this paper, e.g. \cite{levie2018cayleynets,thanou2014graphdict}.
It is worth noting that these cases dominate the state of the art, and we still include other methods in empirical comparisons.

\textbf{Primary observation.}\quad 
From the review of GNN models, we can conclude that usually, the node attributes $X$ and the graph structure matrices $\hat{L}/\hat{A}$ are computed first and then some parameter matrices are applied to obtain the final node representation. 
Starting from this observation, in the following, we extract a linearly approximated framework including GNNs that first construct feature spaces and then apply parameters to reweight them to obtain node representations.

\section{Analysis}
To perform theoretical investigations of the feature space, we abstract a linear approximation of GNNs based on the success of linearisation attempts of \citet{wu2019sgc,xu2018powerful,wang2022powerful}. 
Specifically, we provide a general formulation for the linear approximation of arbitrary graph neural networks. $\overline{\mathtt{GNN}}(X,\hat{A})$ as: 
\begin{align}
    H = \overline{\mathtt{GNN}}(X,\hat{A}) = \sum_{t=0}^{T-1} \textcolor{orange}{\Phi_t(X,\hat{A})} \textcolor{Blu}{\Theta_t},
    \label{eq:linear}
\end{align}
where $\Phi_t(X,\hat{A})\in \mathbb{R}^{n\times d_t}$ is the non-parametric feature space constructing function that inputs the graph data (e.g., node attributes and graph structure) and outputs a feature subspace,
$\Theta\in\mathbb{R}^{d_t\times c}$ is the parameter space to reweight the corresponding feature subspace for each class $c$,
and $T$ is a hyper-parameter of the number of the feature subspaces that the GNN contains. In general, in this linear approximation, a GNN model forms $T$ feature subspaces, i.e., $\Phi_t$, and outputs the addition of all the reweighted subspaces using the respective parameters $\Theta_t$.
Note that the (total) feature space is the union of the subspaces as $\Phi=\{\Phi_t\}_{t=0,1,\cdots,T-1}$. Similarly, we have the (total) parameters $\Theta=\{\Theta_t\}_{t=0,1,\cdots,T-1}$.
Besides, the number of the subspaces $T$ that a GNN model obtains is not parallel with its layer/order, for which we will provide some examples in the later revisiting subsection.

In the following, we will first identify the feature space $\Phi$ and the parameters $\Theta$ for existing GNNs. 
Then, from the perspective of the feature space, we analyze the common mode across different model lines.
Finally, we investigate and summarize the problems behind this mode.

\subsection{Revisiting Existing Models}
\textbf{Spatial GNNs (with non-parametric aggregation).} \quad 
We first transform the recursive formula of spatial GNNs, e.g.,~\eqref{eq:spatial_gnn}, into an explicit formula, by iterating from the initial node attributes that $H^{(0)}=X$ and ignoring the activation function. 
Following Section~\ref{sec:preliminaries}, we consider two examples of spatial GNNs: the vanilla GCN~\citep{kipf2016gcn} and the one with skip-connections~\citep{xu2018powerful}.

The linear approximated explicit formula of a $K$-layer is:
\begin{align}
    H^{(K)} = \textcolor{orange}{\hat{A}^K X} \textcolor{Blu}{\prod_{i=0}^{K-1} W^{(i)}},
    \label{eq:spatial_dl_1}
\end{align}
which forms single feature space $\Phi_0 = \hat{A}^{K} X$ and parameters $\Theta_0=\prod_{i=0}^{K-1} W^{(i)}$ with~$T=1$. 
It is a perfect example that the number of GNN layers $K$ is not identical to the number of the subspaces $T$ a GNN forms.
While \eqref{eq:gin} furthermore considers skip-connections, whose $K$-layer linear approximated explicit formula is formualted as:
\begin{equation}
\begin{aligned}
    H^{(K)} &=  \sum_{i=0}^{K-1} \textcolor{orange}{\hat{A}^{i} X} \textcolor{Blu}{\alpha^{(K-1-i)}W_0^{(K-1-i)} \prod_{j=K-i}^{K-1}W_1^{(j)}}  \\
    &+ \textcolor{orange}{\hat{A}^{K}X}\textcolor{Blu}{\prod_{h=0}^{K-1} W_1^{(h)}}.
    \label{eq:spatial_dl_2}
\end{aligned}
\end{equation}

By this decomposition, this GCN with skip-connections consists of $T=K+1$ feature subspaces.
It forms each feature subspace as $\Phi_t=\hat{A}^tX$.
For the first $T-1$ subspaces, the according respective parameters is denoted as $\Theta_{t;t<T-1}=\alpha^{(K-1-t)}W_0^{(K-1-t)}\prod_{j=K-t}^{T-1}W_1^{(t)}$, and for for the last $\Phi_T$, the parameter is $\Theta_T=\prod_{h=0}^{T-1} W_1^{(h)}$.
Please refer to Appendix~\ref{app:derive_equ6} for the derivation.

\textbf{Spectral GNNs (polynomial-based).}\quad
This type is specified by the explicit formula as \eqref{eq:spectral_gnn}. We remove the activation function, and obtain the linear approximation of a $K$-order spectral GNNs as:
\begin{align}
    H^{(K)} = \sum_{k=0}^K \textcolor{orange}{P_k(\hat{L})X} \textcolor{Blu}{\gamma^{(k)}W^{(0)}}.
    \label{eq:spectral_dl}
\end{align}
We put the learnable polynomial coefficient $\gamma^{(k)}$ together with the parameter matrices. 
We also combine the shared parametric matrices in \eqref{eq:spectral_gnn} as $W^{(0)}=W_1W_2$. 
In this way, \eqref{eq:spectral_dl} forms $T=K+1$ feature subspaces, each denoted as $\Phi_t=P_t(\hat{L})X$, and the parameters used to reweight the respective subspaces are $\Theta_t=W^{(0)}W^{(1)}$. See Appendix~\ref{app:deriv_bern} for further derivations. 

In Table~\ref{tab:gnn_elements_big}, we summarize typical instances of spatial and spectral methods, using different colors to distinguish the feature space $\Phi$ (orange) and the parameters $\Theta$ (blue). 
It shows that the proposed extracted view~\eqref{eq:linear} can support most methods in both spatial and spectral domains. 

\begin{figure}
    \centering 
    \includegraphics[width=\columnwidth]{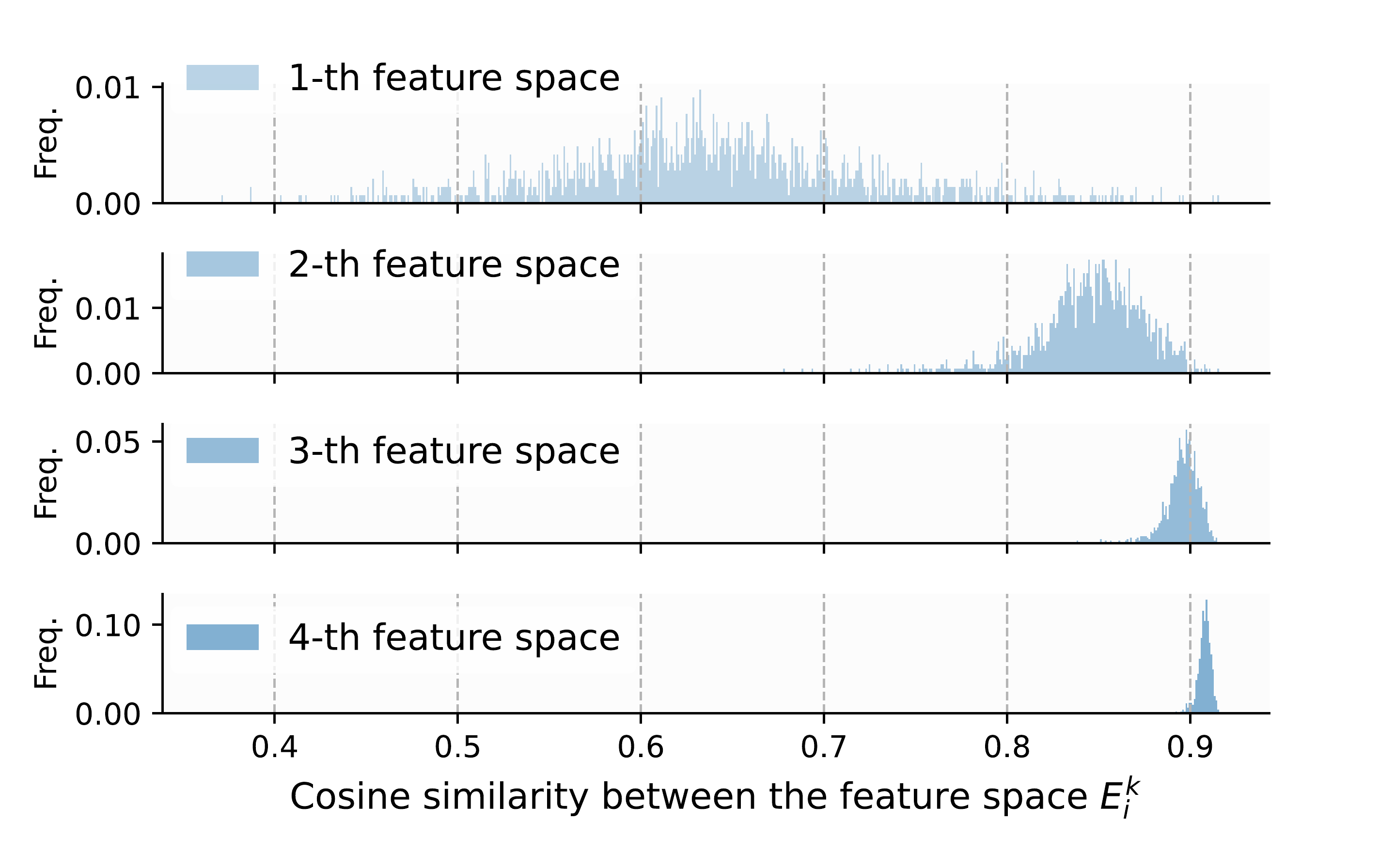}
    \caption{Distribution of the mutual correlation values between the later feature (sub)spaces to the previous total ones.}
    \label{fig:sim_sub}
\end{figure}

\subsection{Analysis of the Feature Space}

In this section, we first give a theoretical argument that the feature subspace of current GNNs obeys asymptotically linear correlation (see Proposition~\ref{theo:1.1}).
Then, we find that the current weight-sharing mechanism weakens the expressiveness of the feature space when the strict linear correlation is satisfied (see Theorem~\ref{theo:1.2}).
In the remainder, we analyze the case where the feature subspaces do not obey strict linear correlation.
We find that the feature space is insufficient when the dimensionality of the node attributes is limited and no assumptions can be made about the feature construction (e.g., heterophily).

Table~\ref{tab:gnn_elements_big} shows that the feature space $\Phi$ in spatial and spectral GNNs is formulated by the multiplication of graph structure matrices function and node attributes, e.g. $\Phi_k=P_k(\hat{L})X$. 
We consider $\hat{A}^kX$ to be the basic element for each feature subspace since other forms such as $P_k(\hat{L})X$ and $\hat{L}$ are all linear transformations of $\hat{A}$. 
Given this, it can be concluded that the feature subspaces of GNNs are sequentially appended as the spatial layers or the spectral orders of GNNs increase, with the latter subspace being the result of applying the aggregation function to the former.

This monotonous construction of the feature space will lead to a high linear correlation between each feature subspace as presented in Proposition~\ref{theo:1.1}. 

\begin{proposition} (its proof can be found in Appendix~\ref{app:theo1.1})
Suppose the feature subspaces are constructed sequentially by $\{\Phi_t=\hat{A}^tX\}_t$. 
As $i \in \mathbb{Z}$ increases, the subspace $\Phi_{t+i}$ gradually tends to be linearly correlated with $\Phi_{t}$.
\label{theo:1.1}
\end{proposition}

To better understand the property of Proposition~\ref{theo:1.1}, we provide a quantitative demonstration using the feature space of GPRGNN~\citep{chien2021gprgnn} as an example. We measure the linear correlation of the appended $k$-th feature (sub)space with the previous ones by calculating the mutual correlation values:
\begin{align}
    E^k_i = \max_{j=0,\cdots,k-1} \mu(\hat{L}^jX, \hat{L}^kX_{\cdot i}),
\end{align} 
where $i$ is the index of the column in $\hat{L}^kX$, and $\mu(M_0,M_1)=\max_{d_u\in M_0,d_v\in M_1}\cos(d_u,d_v)$ is the mutual-coherence of two matrices, based on the cosine distance $\cos$. 
In Figure~\ref{fig:sim_sub}, we visualize the distribution of $\{E^k_i\}$ of all the columns with $k=1,2,3,4$. It confirms that the linear correlation improves significantly with the number of subspaces.
Therefore, both theoretical discussions and visualizations show a trend of increasingly linear correlations between feature subspaces with an increasing number of GNN layers/orders.

Since we are dealing with a gradually linear correlation, in the following we identify two questions about the current feature construction when this condition is strictly fulfilled or not, respectively.

\textbf{Issue 1: Constraint from the weight sharing mechanism.}\quad
From Table~\ref{tab:gnn_elements_big} we see that existing GNNs usually share parameter weights between different subspaces.
Under the condition of linear correlation, we provide a direct argument that using the weight-sharing method limits the expressiveness of the feature space.
\begin{theorem}(its proof can be found in Appendix~\ref{app:theo1.2})
\textcolor{black}{
Suppose $\Phi_a,\Phi_b\in\mathbb{R}^{n\times d}$ are two linearly correlated feature subspaces, i.e. there exists $W_a\in\mathbb{R}^{d \times d}$ such that $\Phi_a W_a = \Phi_b$, and suppose a matrix $B\in \mathbb{R}^{n\times c}$, $c << d$.
If $B$ can be represented by using both subspaces with a common weight $W_B$, i.e., $\gamma_a\Phi_a W_{B}+ \gamma_b \Phi_b W_{B}= B$ and $\gamma_a, \gamma_b \in \mathbb{R}$, 
then $B$ can always be represented by only one subspace $\Phi_a$, i.e., $\Phi_a W_{B}'= B$ and $W_B'\in \mathbb{R}^{d\times c}$.}
\label{theo:1.2}
\end{theorem}
It shows an expressiveness bound of the feature subspaces when they are linearly correlated.
While this linearly dependency can be used as redundancy and is widely used in areas such as dictionary learning~\cite{elad2010sparse}.
In particular, an over-determined linear system can be relaxed to an under-determined one when linearly correlated columns are added to the regressor matrix, making it easier to optimize.
However, the condition of this benefit is not met in existing GNNs due to the weight-sharing mechanism. 
In the next section, we propose a modification to break this constraint.

\textbf{Issue 2: Constraint from limited dimensionality of node attributes.}\quad
Proposition~\ref{theo:1.1} clarifies the tendency of linear correlation between the feature subspaces, but in the first few, this property may not be strictly satisfied. 
It makes the weight-sharing not so flawed and also weakens the effectiveness of the corresponding modification.
In the following, we give a discussion of the condition that the feature spaces are not necessarily linearly correlated, and we consider two limiting scenarios of the dimension $d$ of the node attributes, since the respective discussions are strongly orthogonal.

First, if $d\rightarrow n$, $X$ itself is a sufficient feature space if $X$ is sparse (e.g., bag-of-words features) or each column is linearly independent.
Besides, the weight-sharing method sums different feature subspaces by learnable weights $\gamma$ for each. 
As a result, the optimized sum promotes more of the subspaces that are closer to the labels.
Therefore, the core is the assumptions of subspace construction (e.g., homophilic assumptions and hyperparameter search for aggregation functions) and more flexible polynomial functions (e.g., Chebysheb and Bernstein). 
They all have been widely studied, where the performance of weight-sharing methods is convincing, such as~\cite{wang2022powerful}.

On the other hand, when $d << n$, the node attributes $X$ have a thin shape, making the regression system strictly over-determined.
In addition, without any assumption about the feature space construction, there is hardly an exact solution, \textcolor{black}{even when using linearly correlated copies to perform the expansion (which we will discuss in the next section).}
Therefore, compared to homophilic graphs, heterophilic graphs are a more severe case, especially since there is hardly any assumption about the feature space construction.
For this situation, we propose the other modification below.

\section{Proposed Method}
Here, we first propose \textit{1) feature subspace flattening} for the detrimental condition caused by the weight-sharing mechanism given by Theorem~\ref{theo:1.2} and Proposition~\ref{theo:1.1}.
Then, to compensate for the second case when the dimensionality of the node attributes is limited, we propose \textit{2) structural principal components} to expand the original feature space.

\subsection{Modification 1: Feature Subspaces Flattening}
\textcolor{black}{
For the first problem, we encourage each feature subspace to be independently expressed of each other in the model and weight them separately.}
Before giving supporting proof, we provide an illustration of this modification in Figure~\ref{fig:fagnn}. The benefit of this proposal is given in the following.

\begin{theorem}(its proof can be found in Appendix~\ref{app:theo2.1})
    Suppose $\Phi_a,\Phi_b\in\mathbb{R}^{n\times d}$ are two linearly correlated feature subspaces, i.e., there exits $W_a\in\mathbb{R}^{d \times d}$ such that $\Phi_a W_a = \Phi_b$, and a matrix $B\in \mathtt{R}^{d\times c}$, $c << d$.
    If $B$ can be expressed by $\Phi_a$, i.e., $\Phi_a W_{B}= B$,
    then using both subspaces $\Phi_a$ and $\Phi_b$ independently, i.e., $\Phi_a W_{a} + \Phi_b W_{b} = B$, the optimum is more easily achieved than a weight-sharing style.
    \label{theo:2.1}
\end{theorem} 
It follows from this theorem that feature subspace flattening is more effective than weight-sharng methods when the feature subspaces tend to be linearly correlated. For further discussion on parameter matrices that are stacked (e.g., GCN), please refer to Appendix~\ref{app:stack}, where the same conclusion is maintained.


\begin{figure}[ht]
\begin{center}
\centerline{\includegraphics[width=\columnwidth, trim={0cm 0cm 0 2.2cm}, clip]{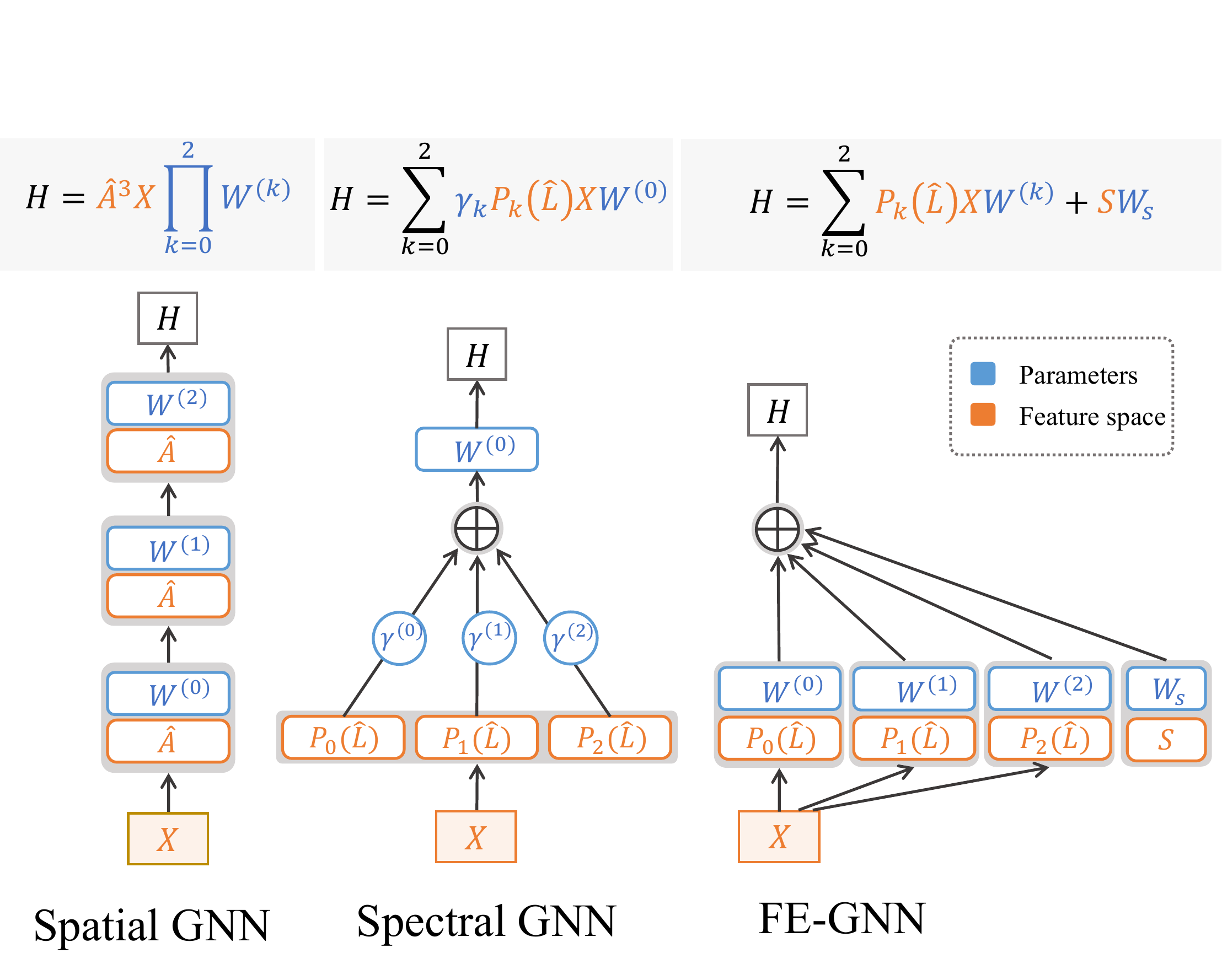}}
\caption{Architecture of our proposal}
\label{fig:fagnn}
\end{center}
\end{figure}

\subsection{Modification 2: Structural Principal Components}
Next, we consider the second issue that the dimension of the node attributes is limited. 
We propose to expand the whole feature space by introducing other feature subspaces.

\begin{table*}[t]
\caption{Overall performance of FE-GNN in node classification}
    \centering
    \vspace{0.1cm}
    \begin{adjustbox}{width=\textwidth}
    \begin{tabular}{ccccccccaaa}
    \toprule
        \multirow{2}{*}{Type} &  \multirow{2}{*}{Baseline}  &  \multirow{2}{*}{Time~(ms)}  & \multicolumn{5}{c}{Homophilic graphs} & \multicolumn{3}{c}{\cellcolor{Gray}Heterophilic graphs}\\
        \cmidrule{4-8} \cmidrule{9-11}
        &&& Cora & CiteSeer & PubMed & Computers & Photo & Squirrel & Chameleon & Actor\\
        \midrule
        \multirow{6}{*}{Spatial} & MLP &-&  76.70$_{\pm\text{0.15}}$ & 76.67$_{\pm\text{0.26}}$ & 85.11$_{\pm\text{0.26}}$ & 82.62$_{\pm\text{0.21}}$ & 84.16$_{\pm\text{0.13}}$ & 37.86$_{\pm\text{0.39}}$ & 57.83$_{\pm\text{0.31}}$ & 38.99$_{\pm\text{0.17}}$ \\
        & GCN &17.42$_{\pm\text{1.64}}$& 87.69$_{\pm\text{0.40}}$ & 79.31$_{\pm\text{0.46}}$ & 86.71$_{\pm\text{0.18}}$ & 83.24$_{\pm\text{0.11}}$ & 88.61$_{\pm\text{0.36}}$ & 47.21$_{\pm\text{0.59}}$ & 61.85$_{\pm\text{0.38}}$ & 28.61$_{\pm\text{0.39}}$ \\
        & GAT &18.06$_{\pm\text{1.18}}$& 88.07$_{\pm\text{0.41}}$ & 80.80$_{\pm\text{0.26}}$ & 86.69$_{\pm\text{0.14}}$ & 82.86$_{\pm\text{0.35}}$ & 90.84$_{\pm\text{0.32}}$ & 33.40$_{\pm\text{0.16}}$ & 51.82$_{\pm\text{1.33}}$ & 33.48$_{\pm\text{0.35}}$ \\
        & GraphSAGE &10.72$_{\pm\text{0.25}}$& 87.74$_{\pm\text{0.41}}$ & 79.20$_{\pm\text{0.42}}$ & 87.65$_{\pm\text{0.14}}$ & 87.38$_{\pm\text{0.15}}$ & 93.59$_{\pm\text{0.13}}$ & 48.15$_{\pm\text{0.45}}$ & 62.45$_{\pm\text{0.48}}$ &  36.39$_{\pm\text{0.35}}$ \\
        & GCNII &8.48$_{\pm\text{0.24}}$& 87.46$_{\pm\text{0.31}}$ & 80.76$_{\pm\text{0.30}}$ & 88.82$_{\pm\text{0.08}}$ & 84.75$_{\pm\text{0.22}}$ & 93.21$_{\pm\text{0.25}}$ & 43.28$_{\pm\text{0.35}}$ & 61.80$_{\pm\text{0.44}}$ & 38.61$_{\pm\text{0.26}}$ \\
        & APPNP &23.74$_{\pm\text{2.08}}$& 87.92$_{\pm\text{0.20}}$ & 81.42$_{\pm\text{0.26}}$ & 88.16$_{\pm\text{0.14}}$ & 85.88$_{\pm\text{0.13}}$ & 90.40$_{\pm\text{0.34}}$ & 39.63$_{\pm\text{0.49}}$ & 59.01$_{\pm\text{0.48}}$ & 39.90$_{\pm\text{0.25}}$ \\
        \midrule
        \multirow{3}{*}{Spectral} & ChebyNet &20.26$_{\pm\text{1.03}}$& 87.17$_{\pm\text{0.19}}$ & 77.97$_{\pm\text{0.36}}$ & 89.04$_{\pm\text{0.08}}$ & 87.92$_{\pm\text{0.13}}$ & 94.58$_{\pm\text{0.11}}$ & 44.55$_{\pm\text{0.28}}$ & 64.06$_{\pm\text{0.47}}$ & 25.55$_{\pm\text{1.67}}$\\
        & GPRGNN &23.55$_{\pm\text{1.26}}$& 87.97$_{\pm\text{0.24}}$ & 78.57$_{\pm\text{0.31}}$ & 89.11$_{\pm\text{0.08}}$ & 86.07$_{\pm\text{0.14}}$ & 93.99$_{\pm\text{0.11}}$ & 43.66$_{\pm\text{0.22}}$ & 63.67$_{\pm\text{0.34}}$ & 36.93$_{\pm\text{0.26}}$ \\
        & BernNet &36.88$_{\pm\text{0.84}}$& 87.66$_{\pm\text{0.26}}$ & 79.34$_{\pm\text{0.32}}$ & 89.33$_{\pm\text{0.07}}$ & 88.66$_{\pm\text{0.08}}$ & 94.03$_{\pm\text{0.08}}$ & 44.57$_{\pm\text{0.33}}$ & 63.07$_{\pm\text{0.43}}$ & 36.89$_{\pm\text{0.30}}$\\
        \midrule
        \multirow{5}{*}{Unified} & GNN-LF &52.77$_{\pm\text{4.50}}$& 88.12$_{\pm\text{0.06}}$ & {\bf83.66$_{\pm\text{0.06}}$} & 87.79$_{\pm\text{0.05}}$ & 87.63$_{\pm\text{0.05}}$ & 93.79$_{\pm\text{0.06}}$ & 39.03$_{\pm\text{0.08}}$ & 59.84$_{\pm\text{0.09}}$ & 41.97$_{\pm\text{0.06}}$\\
        & GNN-HF &53.28$_{\pm\text{4.51}}$& 88.47$_{\pm\text{0.09}}$ & 83.56$_{\pm\text{0.10}}$ & 87.83$_{\pm\text{0.10}}$ & 86.94$_{\pm\text{0.06}}$ & 93.89$_{\pm\text{0.10}}$ & 39.01$_{\pm\text{0.51}}$ & 63.90$_{\pm\text{0.11}}$ & {\bf42.47$_{\pm\text{0.07}}$}\\
        & ADA-UGNN &14.36$_{\pm\text{0.21}}$& 88.92$_{\pm\text{0.11}}$ & 79.34$_{\pm\text{0.09}}$ & 90.08$_{\pm\text{0.05}}$ & 89.56$_{\pm\text{0.09}}$ & 94.66$_{\pm\text{0.07}}$ & 44.58$_{\pm\text{0.16}}$ & 59.25$_{\pm\text{0.16}}$ & 41.38$_{\pm\text{0.12}}$\\
        \cmidrule{2-11}
        & {\bf FE-GNN~(C)} &15.8$_{\pm\text{0.11}}$& {\bf89.45$_{\pm\text{0.22}}$} & 81.96$_{\pm\text{0.23}}$ & \textbf{90.27$_{\pm\text{0.49}}$} & {\bf90.79$_{\pm\text{0.08}}$} & 95.36$_{\pm\text{0.14}}$ & 67.82$_{\pm\text{0.26}}$ & {\bf73.33$_{\pm\text{0.35}}$} & 40.54$_{\pm\text{0.15}}$ \\
        & {\bf FE-GNN~(M)} &14.6$_{\pm\text{0.32}}$& 89.09$_{\pm\text{0.22}}$ & 81.76$_{\pm\text{0.23}}$ & 89.93$_{\pm\text{0.23}}$ & 90.60$_{\pm\text{0.11}}$ & {\bf95.45$_{\pm\text{0.15}}$} & {\bf67.90$_{\pm\text{0.23}}$} & 73.26$_{\pm\text{0.38}}$ & 40.91$_{\pm\text{0.22}}$ \\
        \bottomrule
    \end{tabular}
    \end{adjustbox}
    \label{tab:fegnn_overall}
\end{table*}

There are two criteria that we consider. 
One is that the introduced feature subspace should be highly uncorrelated with the existing feature subspace, otherwise, according to the analysis in the previous section, it may not be the same as the previously proposed modification under this condition.
The other is that the dimension of the introduced subspace should not be too large, otherwise noise might be included and the computational cost would also increase.

Considering this, for graph-structured data, there are two data matrices, i.e., node attributes and graph structure matrices, and the former has been explicitly exploited as one of the feature subspaces in most GNN models, as summarized in Table~\ref{tab:gnn_elements_big}. 
On the contrary, structure matrices are used only in combination with node attributes.
Given these conditions, we propose to construct the expansion subspace using the truncated SVD of the structural matrix, called \textit{structural principal components} as modification 2.
It naturally provides orthogonal subspaces, and the truncated form limits the dimension of the matrices. Thus, two criteria are met.
Specifically, we extract the low-dimensional information for the graph structure to obtain its principal components: 
\begin{align}
    {S = \tilde{Q}\tilde{V}};\hat{A}=QVR^T,
    \label{eq:fegnn_svd}
\end{align}
where $\tilde{Q},\tilde{V}$ are the principal components and the corresponding singular values. 
Besides, we prove the effectiveness of this modification in the following theorem.

\begin{theorem}(its proof can be found in Appendix~\ref{app:theo2.2})
\textcolor{black}{
Suppose the dimensionality of the node attributes is much smaller than the number of nodes, i.e., $d<< n, X\in \mathbb{R}^{n\times d}$, and a $z$-truncated SVD of $\hat{L}$, which satisfies $||U_zS_z-\hat{L}||_2 < \epsilon$, where $\epsilon$ is a sufficiently small constant. 
Then the linear system $(\Phi_k,U_zS_z) W_B' = B$ can achieve a miner error than the linear system $\Phi_k W_B = B$.}
\label{theo:2.2}
\end{theorem}

So far, we have introduced two modifications, and together they contribute to a new proposed GNN model called Feature Expanded Graph Neural Network~(FE-GNN). It is shown in Figure~\ref{fig:fagnn} and formulated as follows,
\begin{align}
    H = \sum_{k=0}^{K} \textcolor{orange}{P_k(\hat{L})X} \textcolor{Blu}{W^{(k)}} + \textcolor{orange}{S}\textcolor{Blu}{W_s}. 
    \label{eq:wgnn}
\end{align}
It constructs the feature space in two ways. 
The first part inherits the polynomial-based GNNs and takes advantage of the multiplication of the polynomials of the structural matrix $P_k(\hat{L})$ and the node attributes $X$. 
Second, we use the principal components of the structural matrices to form another feature subspace $S$.
In addition, FE-GNN uses independent parameter matrices $W^{(k)}$ and $W_s$ to reweight each feature subspace to provide flexible reweighting. 

In particular, the $\Phi_k$ and $S$ feature spaces can have unbalanced scales and lead to poor reweighting.
Therefore, we add a column-wise normalization to ensure that each column of each feature subspace contributes equally to the whole feature space. 
Finally, to better verify the importance of the feature space view, our implementation is purely linear without any deep artifices such as activation functions or dropout except for the cross-entropy loss, while we obey the original implementation for the baselines and their non-linearity is preserved.

\subsection{Discussion}
\label{sec:discussion}

Our analysis and proposal make few assumptions about the distribution of node attributes, the graph structure, or even their relationships. 
Since the point view of the feature space is not an analysis method proposed specifically for graph structure data or GNNs, our investigations are more general and jump out of the limitations of other views.
For example, graph denoising~\cite{zhu2021interpreting} and spectral graph theory view~\cite{balcilar2021bridge} both ignore the property of node attributes, which is the key for our second proposal, and instead focus on the transformations of the structure matrices.
We provide a more comprehensive comparison of related work in section~\ref{sec:relate}.

In our proposed method, the first modification that flattens the feature subspace improves the effectiveness of each feature subspace, but the number of parameters must be higher because no parameter matrices are shared. In experiments, we surprisingly found relatively low training time costs compared to baselines.
Furthermore, the second modification can be misinterpreted as an arbitrary addition of structural information. With this in mind, we will conduct additional experiments with other information extraction methods.
Besides, in our approach, the aggregation process can be abstracted as preprocessing to further speed up the training process. We leave this as future work; in our experiments, aggregations are computed during each training session for a fair comparison. 

\textcolor{black}{Finally, it is worth noting that the linear approximation is adopted for the non-linear GNNs for the following reasons: 1) it allows us the convenience of a deeper view of the GNN models, 2) linearization is a quite normal setting in the theoretical analysis in general machine learning and deep learning analysis, since the non-linearity can hardly provide rigorous arguments~\cite{wu2019simplifying, wesley2021nonlinear}, and 3) the proposed model is fully consistent with the analytical view of linearity.}

\begin{figure}[t]
    \centering
    \begin{minipage}{0.5\textwidth}    
    \centering
    \includegraphics[width=\textwidth]{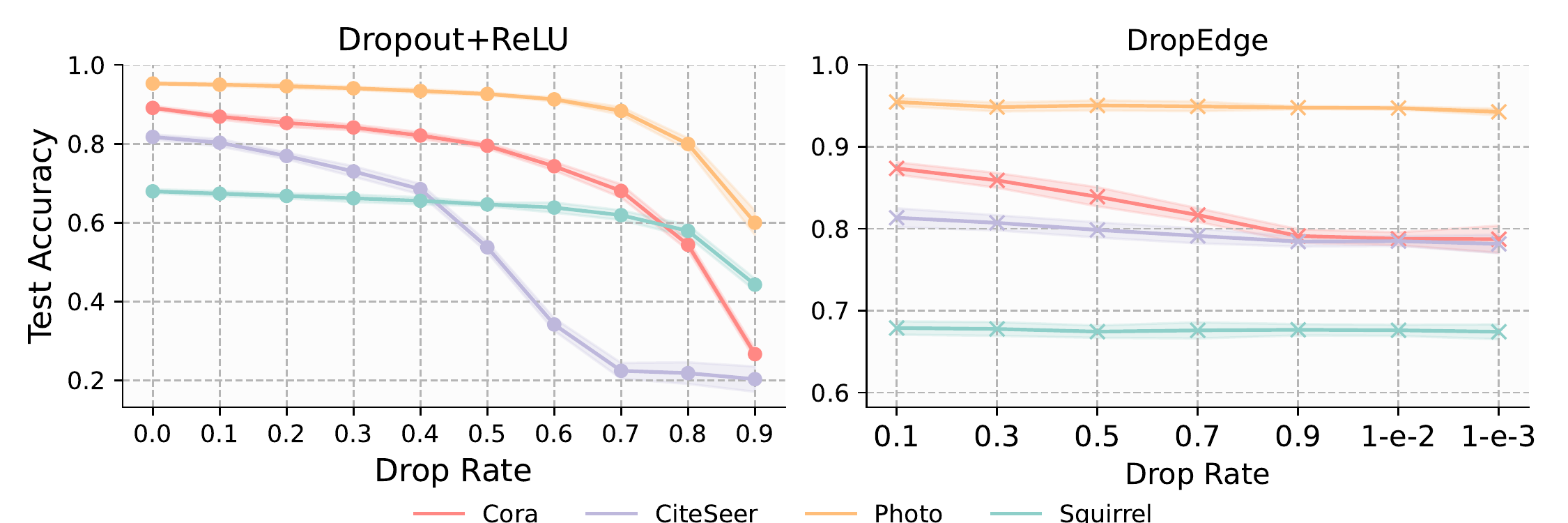}
    \caption{\small Analysis of the deep artifices on FE-GNN.}
    \label{fig:fegnn_deep_artifices}
    \end{minipage}
\end{figure}

\section{Other Related Work}
\label{sec:relate}
\textbf{Unified perspectives for GNNs.}\quad
There have been several perspectives for studying GNNs representations.
\citet{balcilar2021bridge} first bridge the spatial methods to the spectral ones, that they assign most of the spatial GNNs with their corresponding graph filters. 
More specifically, they begin with GNN models' convolution matrix and then summarize their frequency responses. 
%
\citet{unifiedgnn_ma} regard the aggregation progress of GCN~\citep{kipf2016gcn}, GAT~\citep{velickovic2018gat}, and APPNP~\citep{klicpera2019appnp} as graph signal denoising problem, which aims to recover a clean signal from an objective and a certain denoising regularization term.
Given this, the authors consider generalize the smoothing regularization term to and propose ADA-UGNN.
However, it also ignores the property of node attributes but focusing on the flexibility of the denoising functions.
\citet{zhu2021interpreting} give a more comprehensive summary of GNNs from an optimization view, which partly overlaps with \cite{unifiedgnn_ma}'s opinions of graph signal denoising. 
They propose GNN-LF/HF with adjustable objective that behaves as a low/high-pass filter.


\begin{figure}
    \centering
    \includegraphics[width=0.8\columnwidth]{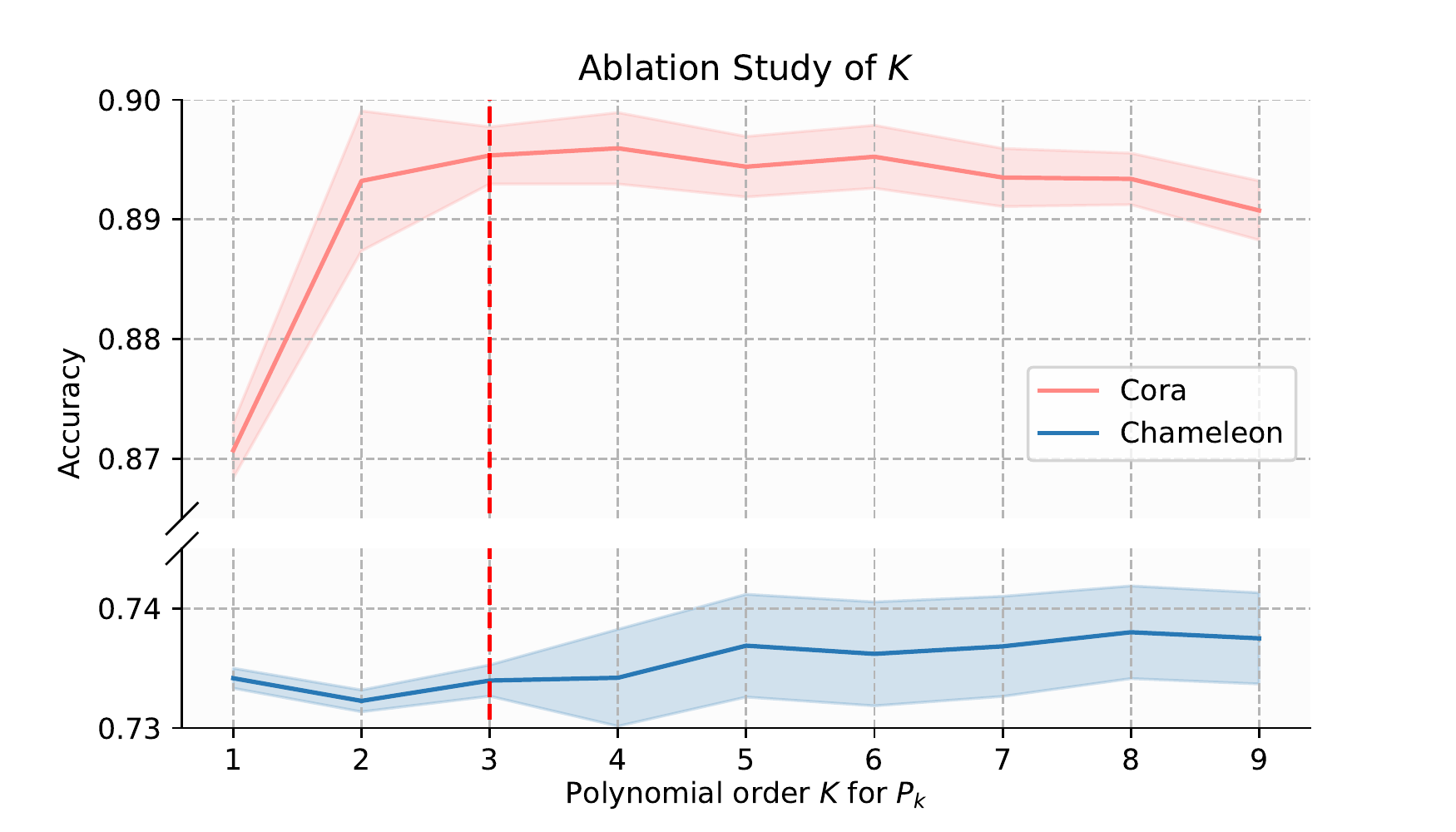}
    \caption{Ablation study of  $K$}
    \label{fig:fegnn_oversmoothing}
\end{figure}

\begin{figure}
    \centering
    \includegraphics[width=0.8\columnwidth, trim={0 0cm 0 0.4cm 0.cm}, clip]{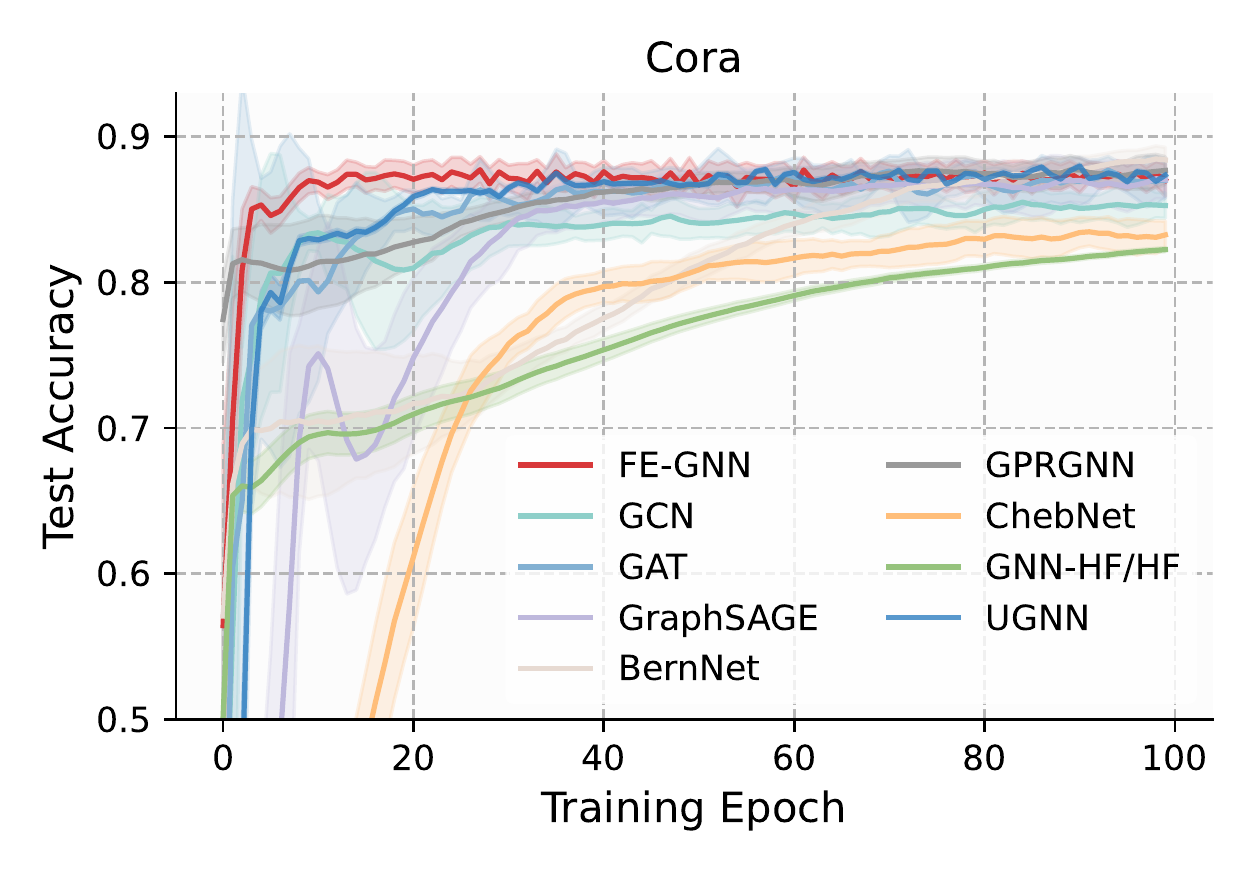}
    \caption{Convergence curve}
    \label{fig:fegnn_converge}
\end{figure}

\textbf{Trend of flattening feature subspaces.}\quad
In addition to these unified investigations of GNN representations, another line of related literature is the trend of flattening feature subspaces and the inclusive of graph structure matrices.
Initial connection~\cite{li2019deepgcns, chen2020simple} is a partial of our first modification that disentangles the dependency between the first two subspaces.
There also methods include graph structure information independently, such as LINK~\cite{lim2021large} includes the whole adjacency matrix directly bearing a high parameter consumption.
Distance encoding technique encodes the local position using graph structure matrices to supplementary node features~\cite{dwivedi2020benchmarking}, however, the point view of feature space is untouched.
\textcolor{black}{\citet{maurya2021simplifying} also adopt the independent weight for each feature subspace. However, they perform feature selection at the subspace level, while our proposal treats all columns of each subspace equally. Moreover, FE-GNN specifically treats the graph structure as an information source instead of carrying out message-passing. More empirical comparisons are shown in Table~\ref{tab:rebuttal_fsgnn} in Appendix~\ref{app:rebuttal_fsgnn}.}

Generally, there have been some scattered attempts to verify the extension perspective of the feature space in the existing work, but due to the lack of a proposal for this view, no more essential conclusions have been drawn. 

\textbf{Over-smoothing.}\quad
\textcolor{black}{Similar to the linear correlation between each feature subspace (see Proposition~\ref{theo:1.1}), the strong similarity between node embeddings has been studied extensively, e.g., in the over-smoothing problem~\cite{li2018deeper, huang2020tackling, cai2020note, sun2022clar}. However, our view is from a column-wise perspective (i.e., hidden dimension) and discusses the correlation between the columns of the feature subspaces, while over-smoothing investigations usually focus on the row-wise perspective (i.e., node dimension) and consider the similarity of the output representations of the nodes. They are obviously different, and a further study of these two should also be an interesting future work.}


\section{Experiments}
\label{sec:experiments}
We evaluate FE-GNN\footnote{Our code is available at \url{https://github.com/sajqavril/Feature-Extension-Graph-Neural-Networks.git}} on: 
(1) node classification results,
(2) ablation studies, and
(3) efficiency check.

\textbf{Dataset.}\quad 
We implement our experiments on homophilic datasets, i.e., Cora, CiteSeer, PubMed, Computers, and Photo~\citep{yang2016revisiting, shchur2018pitfalls}, and heterophilic Chameleon, Squirrel, and Actor~\citep{rozemberczki2021multi, pei2020geom}. Among them, Chameleon and Squirrel are two cases that bear problem 2, i.e., limited dimension of node attributes and no homophilic assumptions to use.  
More details can be found in the appendix~\ref{app:settings}.



\textbf{Baselines.}\quad 
We compare a list of state-of-the-art GNN methods.
For spatial GNNs we have GCN~\citep{kipf2016gcn}, GAT~\citep{velickovic2018gat}, GraphSAGE~\citep{hamilton2017inductive}, GCNII~\citep{chen2020simple} and APPNP~\citep{klicpera2019appnp}, where MLP is included as a special case.
For spectral methods we take ChebyNet~\citep{defferrard2016convolutional}, GPRGNN~\citep{chien2021gprgnn} and BernNet~\citep{he2021bernnet}.
We also cover the recent unified models, ADA-UGNN~\citep{unifiedgnn_ma} and GNN-LF/HF~\citep{zhu2021interpreting}.
FE-GNN uses the Chebyshev or monomial polynomials to construct the feature space, and we refer to the corresponding versions as FE-GNN~(C) and FE-GNN~(M), respectively.
Please refer to Appendix~\ref{appexp:hyper} for more details on the implementation.

\subsection{Node Classification}
\label{sec: node classification}
We test on the node classification task with random 60\%/20\%/20\% splits and summarize the results of 100 runs in Table~\ref{tab:fegnn_overall}, reporting the average accuracy with a 95\% confidence interval. 
We observe that FE-GNN has almost the best performance on homophilic graphs. In particular, compared to the current SoTA method ADA-UGNN~\citep{unifiedgnn_ma}, which unifies the objectives in both spatial and spectral domains, our FE-GNN achieves a $1.1\%$ accuracy improvement on average on 5 homophilic graph datasets. 
We attribute the superior performance of GNN-LF/HF on Citeseer to its complex hyperparameter tuning, where more fine-grained constraints of the parameters can be found, which we consider as future work in section~\ref{sec:conclusion}.
Furthermore, FE-GNN achieves on average ${32.0\%}$ improvement on three heterophilic graph datasets than the GCN baseline. 
It is worth highlighting the huge margin of FE-GNN over others on Chameleon and Squirrel, where they perfectly fit our assumption of the second modification, i.e., heterophilic and limited dimensionality of node attributes. The results of the ablation studies are also consistent with this.

\begin{table}[t]
    \centering\vspace{-0.2cm}
    \caption{Ablation study of 1) flattening feature subspaces}
    \begin{adjustbox}{width=0.95\columnwidth}
    \begin{tabular}{lccccc}
        \toprule
        & Cora & CiteSeer & Chameleon & Squirrel \\
        \midrule
         k=2 & $89.15 \pm_{0.86} $ & $81.97 \pm_{1.10} $ & $73.41 \pm_{1.40} $ & $67.37 \pm_{0.83}$  \\
         k=2 (WS)  & $87.21 \pm_{0.83} $ & $78.39 \pm_{0.72} $ & $72.94 \pm_{1.22} $ & $66.01 \pm_{1.31}$ \\ 
         k=4  & $88.56 \pm_{2.01} $ & $80.19 \pm_{0.84} $ & $73.27 \pm_{1.56} $ & $67.40 \pm_{0.90}$ \\
         k=4 (WS)  & $87.28 \pm_{1.38} $ & $77.72 \pm_{0.86} $ & $73.23 \pm_{1.72} $ & $66.43 \pm_{1.72}$ \\
         k=8  & $88.92 \pm_{0.88} $ & $81.11 \pm_{0.89} $ & $73.85 \pm_{1.52} $ & $67.93 \pm_{2.04}$ \\
         k=8 (WS)  & $86.92 \pm_{1.66} $ & $77.32 \pm_{0.33} $ & $73.15 \pm_{1.83} $ & $66.63 \pm_{2.38}$ \\
         k=16  & $88.26 \pm_{0.14} $ & $80.54 \pm_{1.03} $ & $73.88 \pm_{1.53} $ & $67.82 \pm_{1.54}$\\
         k=16 (WS)  & $87.34 \pm_{1.98} $ & $78.60 \pm_{0.70} $ & $72.94 \pm_{1.79} $ & $66.65 \pm_{2.15}$\\
        \bottomrule
    \end{tabular}
    \end{adjustbox}
    \label{tab:fegnn_share}
\end{table}

\begin{table*}[t]
    \centering 
     \caption{Time consumption of SVD}
    \begin{adjustbox}{width=0.8\textwidth}
    \begin{tabular}{cccccc}
    \toprule
            & Cora & CiteSeer & Chameleon & Squirrel & Actor \\
         \midrule
        Training time (ms) & $4000.01 \pm 52.23$ & $4103.46 \pm 133.77$ & $2818.21 \pm 81.87$ & $6096.43 \pm 403.95$ & $6074.39 \pm 547.34$ \\
        SVD time (ms) & $3.88 \pm 0.08$ & $8.76 \pm 0.05$ & $61.80 \pm 0.24$ & $432.43 \pm 0.70$ & $3.99 \pm 0.02$ \\
        \# of epochs & 252 & 252 & 252 & 271 & 252 \\
        \cellcolor{Gray}SVD time / Training time (\%) & \cellcolor{Gray}0.097 & \cellcolor{Gray}0.21 & \cellcolor{Gray}2.2 & \cellcolor{Gray}7.1 & \cellcolor{Gray}0.066\\
        \bottomrule
    \end{tabular}
    \end{adjustbox}
    \label{tab:fegnn_svd_time}
\end{table*}



\begin{table}[h]
    \centering
    \caption{Ablation study of 2) structural principal components}
    \begin{adjustbox}{width=\columnwidth}
    \begin{tabular}{cccccc}
        \toprule
        & Cora & CiteSeer & PubMed & Squirrel & Chameleon \\
        \midrule
        \cellcolor{Gray}FE-GNN(C) & \cellcolor{Gray}{\bf89.45$_{\pm\text{0.22}}$} & \cellcolor{Gray}{\bf81.96$_{\pm\text{0.23}}$} & \cellcolor{Gray}89.87$_{\pm\text{0.49}}$ & \cellcolor{Gray}67.82$_{\pm\text{0.26}}$ & \cellcolor{Gray}{\bf73.33$_{\pm\text{0.35}}$}\\
        \cellcolor{Gray}FE-GNN(M) & \cellcolor{Gray}89.09$_{\pm\text{0.22}}$ & \cellcolor{Gray}81.76$_{\pm\text{0.23}}$ & \cellcolor{Gray}89.93$_{\pm\text{0.23}}$ & \cellcolor{Gray}{\bf67.90$_{\pm\text{0.23}}$} & \cellcolor{Gray}73.26$_{\pm\text{0.38}}$ \\
        w/o norm & 86.23$_{\pm\text{1.43}}$ & 79.32$_{\pm\text{0.59}}$ & {\bf90.27$_{\pm\text{0.49}}$} & 64.70$_{\pm\text{1.10}}$ & 68.25$_{\pm\text{1.64}}$ \\
        w/o $S$ & 89.20$_{\pm\text{0.93}}$ & 81.95$_{\pm\text{0.87}}$ & 89.76$_{\pm\text{0.46}}$ & 43.21$_{\pm\text{0.99}}$ & 61.54$_{\pm\text{1.52}}$ \\
        w/o $P_k(\hat{L})X_{k>0}$ & 71.10$_{\pm\text{1.72}}$ & 74.38$_{\pm\text{1.01}}$ & 86.61$_{\pm\text{0.54}}$ & 67.90$_{\pm\text{0.96}}$ & 73.35$_{\pm\text{1.21}}$ \\
        w/o $P_k(\hat{L})X_{k=0}$ & 84.70$_{\pm\text{1.05}}$ & 58.60$_{\pm\text{2.19}}$ & 85.84$_{\pm\text{0.45}}$ & 65.75$_{\pm\text{0.63}}$ & 72.61$_{\pm\text{1.60}}$ \\
        \bottomrule
    \end{tabular}
    \end{adjustbox}
    \label{tab:fegnn_ablation}
\end{table}

\subsection{Ablation Studies}
We study the contribution of different components and hyperparameter effects in FE-GNN. 

\textbf{Does the feature flattening work? Yes.
}\quad
In Table~\ref{tab:fegnn_share}, we compare our proposal with a weight-sharing~(WS) instance where the principal components are retained. 
It shows that over an increasing number of feature subspaces, the flattening feature subspace is always observed better performance, which verifies the discussion of theorems~\ref{theo:1.2} and ~\ref{theo:2.1}.

\textbf{When do structural principal components work? On limited node attributes and the heterophilic case.
}\quad
We evaluate FE-GNN on 5 datasets of both homophilic and heterophilic graphs in 3 different feature space constructions: including $w/o\ S$, $ w/o\ P_k(\hat{L})X_{k=0}$, and $ w/o\ P_k(\hat{L})X_{k>0}$, which denote the feature space constructions without graph structure, without note attributes, and without the combination of both.
In the ablation results of Table~\ref{tab:fegnn_ablation}, we found that $w/o\ S$ works well on homophilic graphs, but fails on 
heterophilic ones (with limited node attribute dimension), while the other two work inversely. 
Meanwhile, $w/\ S$ greatly improves the bounded cases, and these results confirm our original intention of proposing structural principal components.
In addition, we provide other variants in Appendix~\ref{app:other_sj} that include structural information as an extension of our discussion in Section~\ref{sec:discussion}, and row-wise normalization in Appendix~\ref{app:atom}, where our proposal is comparably effective.


\textbf{What proportion of the truncated SVD is appropriate?  94\%.}\quad
We use different ratios of singular vectors and values to construct $S$, i.e., the top $j$ singular values get $r=\sum_{j=1}^{i}V_{jj}/\sum_{j=1}^{n}V_{jj}$ ratio of components. 
In Figure~\ref{fig:fegnn_svd}, we show the accuracy of CiteSeer and Chameleon as the ratio of singular values increases, where CiteSeer is robust to variation in the ratio, while Chameleon performs best at $r=94\%$. So we use 94\% for further experiments.
We provide more SVD results in the Appendix~\ref{app:svd}.

\textbf{At what order is a polynomial sufficient? Around three.}\quad 
We test a progressive order $K$ of polynomials on Cora and Chameleon as shown in Figure~\ref{fig:fegnn_oversmoothing}.
The performance in Cora increases from $1$ to $3$ and decreases slightly, while Chameleon shows only small changes.
This suggests that order 3 is good enough to achieve near-optimal performance, which is also an advance by flattening the feature subspaces compared to deep GNNs. A more comprehensive comparison can be found in Appendix~\ref{app:best_k}.

\textbf{Is linearity sufficient for constructing features? Yes.}\quad
We apply the nonlinear activation function to FE-GNN and some deep learning training techniques, including Dropout~\cite{agarap2018relu} and DropEdge~\cite{rong2020dropedge}, to verify the sufficiency of our linear implementation.
In Figure~\ref{fig:fegnn_deep_artifices}, we show the corresponding performance with different drop ratios on four datasets, where both show worse performance when increasing the drop rate. 
Therefore, our linear construction is sufficient for FE-GNN.
These results also argue for more attention to feature space construction to avoid over-application of deep artifices.

\subsection{Efficiency Check}
Finally, we examine the efficiency of our proposal, including the training time cost and the truncated SVD time.
In Table~\ref{tab:fegnn_overall}, we collect the training time per epoch~(ms) for each method, which shows that FE-GNN behaves at a comparable time cost to other baselines, \textcolor{black}{even though it bears more computational cost from the independent feature expression}. 
Note that the time we report includes graph propagation for a fair comparison, although FE-GNN can reduce it further by constructing the feature space in a preprocessing manner. 
In Figure~\ref{fig:fegnn_converge} we compare the convergence time for all methods and observe that FE-GNN consumes the minimum number of training epochs while achieving the highest accuracy\textcolor{black}{, which exactly reveals the conclusion of easier optimization progress from Theorem~\ref{theo:2.1}.}
And in table ~\ref{tab:fegnn_svd_time}, we show the training time and SVD time (as preprocessing), from which we find that the SVD time in the total training time is less than 10\%, which confirms the applicability. 

\begin{figure}[t]
    \centering
    \begin{minipage}{\columnwidth}
    \centering
    \includegraphics[width=\textwidth, trim={0.5cm 0.4cm 0.5cm 0.2cm}, clip]{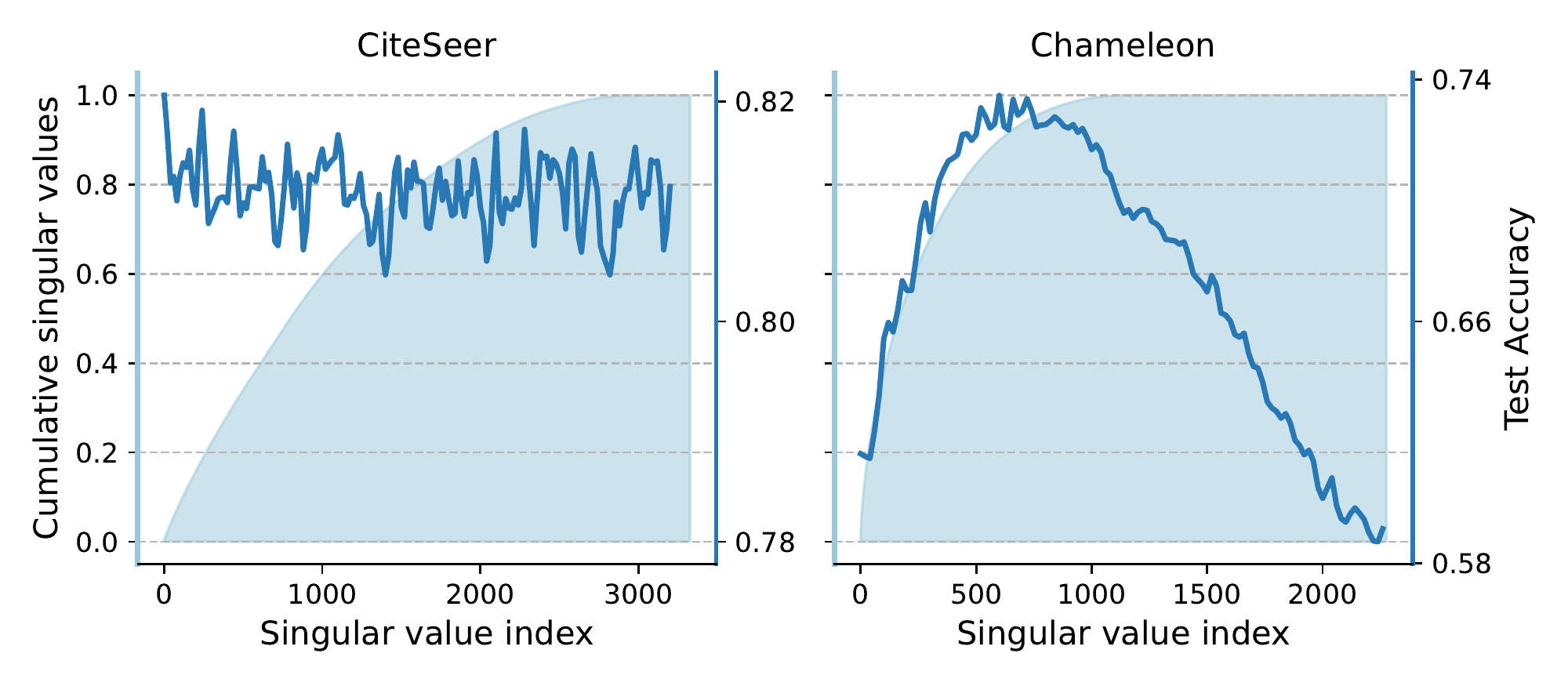}
    \caption{\small
    Sensitivity study of truncated SVD}\vspace{-0.3cm}
    \label{fig:fegnn_svd}
    \end{minipage}
\end{figure} 

\section{Conclusion}
\label{sec:conclusion}
In this paper, we provide the feature space view to analyze GNNs, which separates the feature space and the parameters. 
Together, we provide a theoretical analysis of the existing feature space of GNNs and summarize two issues.
We propose 1) feature subspace flattening and 2) structural principal components for these issues.
Extensive experimental results verify their superiority. 

\textbf{Limitations}\quad 
Nonlinear cases are not included in our work and will be considered in future work.
Also, the correlation between the subspaces should be studied more carefully beyond the linear correlation property; in a sense, the parameters can be further reduced by introducing reasonable constraints.
Finally, more feature space construction methods should be discovered for future work.
\section*{Acknowledgements}
This work was partly supported by the National Key Research and Development Program of China (No. 2020YFB1708200) ,  the "Graph Neural Network Project" of Ping An Technology (Shenzhen) Co., Ltd. and AMiner.Shenzhen SciBrain fund.
This work was also supported in part by the NSF-Convergence Accelerator Track-D award \#2134901, by the National Institutes of Health (NIH) under Contract R01HL159805, by grants from Apple Inc., KDDI Research, Quris AI, and IBT, and by generous gifts from Amazon, Microsoft Research, and Salesforce.

\newpage
\bibliography{icml2023/wgnn_icml23/bib}

\begin{thebibliography}{45}
\providecommand{\natexlab}[1]{#1}
\providecommand{\url}[1]{\texttt{#1}}
\expandafter\ifx\csname urlstyle\endcsname\relax
  \providecommand{\doi}[1]{doi: #1}\else
  \providecommand{\doi}{doi: \begingroup \urlstyle{rm}\Url}\fi

\bibitem[Agarap(2018)]{agarap2018relu}
Agarap, A.~F.
\newblock Deep learning using rectified linear units (relu).
\newblock \emph{arXiv preprint arXiv:1803.08375}, 2018.

\bibitem[Balcilar et~al.(2021)Balcilar, Renton, H{\'{e}}roux,
  Ga{\"{u}}z{\`{e}}re, Adam, and Honeine]{balcilar2021bridge}
Balcilar, M., Renton, G., H{\'{e}}roux, P., Ga{\"{u}}z{\`{e}}re, B., Adam, S.,
  and Honeine, P.
\newblock Analyzing the expressive power of graph neural networks in a spectral
  perspective.
\newblock In \emph{9th International Conference on Learning Representations,
  {ICLR} 2021, Virtual Event, Austria, May 3-7, 2021}, 2021.

\bibitem[Bengio et~al.(2013)Bengio, Courville, and
  Vincent]{bengio2013representation}
Bengio, Y., Courville, A., and Vincent, P.
\newblock Representation learning: A review and new perspectives.
\newblock \emph{IEEE transactions on pattern analysis and machine
  intelligence}, 35\penalty0 (8):\penalty0 1798--1828, 2013.

\bibitem[Bianchi et~al.(2021)Bianchi, Grattarola, Livi, and
  Alippi]{bianchi2021arma}
Bianchi, F.~M., Grattarola, D., Livi, L., and Alippi, C.
\newblock Graph neural networks with convolutional arma filters.
\newblock \emph{IEEE Transactions on Pattern Analysis and Machine
  Intelligence}, 2021.

\bibitem[Cai \& Wang(2020)Cai and Wang]{cai2020note}
Cai, C. and Wang, Y.
\newblock A note on over-smoothing for graph neural networks.
\newblock \emph{arXiv preprint arXiv:2006.13318}, 2020.

\bibitem[Chen et~al.(2020)Chen, Wei, Huang, Ding, and Li]{chen2020simple}
Chen, M., Wei, Z., Huang, Z., Ding, B., and Li, Y.
\newblock Simple and deep graph convolutional networks.
\newblock In \emph{International Conference on Machine Learning}, pp.\
  1725--1735. PMLR, 2020.

\bibitem[Chien et~al.(2021)Chien, Peng, Li, and Milenkovic]{chien2021gprgnn}
Chien, E., Peng, J., Li, P., and Milenkovic, O.
\newblock Adaptive universal generalized pagerank graph neural network.
\newblock In \emph{9th International Conference on Learning Representations,
  {ICLR} 2021, Virtual Event, Austria, May 3-7, 2021}, 2021.

\bibitem[Chung \& Graham(1997)Chung and Graham]{chung1997spectral}
Chung, F.~R. and Graham, F.~C.
\newblock \emph{Spectral graph theory}.
\newblock American Mathematical Soc., 1997.

\bibitem[Defferrard et~al.(2016)Defferrard, Bresson, and
  Vandergheynst]{defferrard2016convolutional}
Defferrard, M., Bresson, X., and Vandergheynst, P.
\newblock Convolutional neural networks on graphs with fast localized spectral
  filtering.
\newblock \emph{Advances in neural information processing systems}, 29, 2016.

\bibitem[Dwivedi et~al.(2020)Dwivedi, Joshi, Laurent, Bengio, and
  Bresson]{dwivedi2020benchmarking}
Dwivedi, V.~P., Joshi, C.~K., Laurent, T., Bengio, Y., and Bresson, X.
\newblock Benchmarking graph neural networks.
\newblock \emph{arXiv preprint arXiv:2003.00982}, 2020.

\bibitem[Elad(2010)]{elad2010sparse}
Elad, M.
\newblock \emph{Sparse and redundant representations: from theory to
  applications in signal and image processing}, volume~2.
\newblock Springer, 2010.

\bibitem[Fan et~al.(2019)Fan, Ma, Li, He, Zhao, Tang, and Yin]{fan2019graph}
Fan, W., Ma, Y., Li, Q., He, Y., Zhao, E., Tang, J., and Yin, D.
\newblock Graph neural networks for social recommendation.
\newblock In \emph{The world wide web conference}, pp.\  417--426, 2019.

\bibitem[Hamilton et~al.(2017)Hamilton, Ying, and
  Leskovec]{hamilton2017inductive}
Hamilton, W., Ying, Z., and Leskovec, J.
\newblock Inductive representation learning on large graphs.
\newblock \emph{Advances in neural information processing systems}, 30, 2017.

\bibitem[He et~al.(2021)He, Wei, Xu, et~al.]{he2021bernnet}
He, M., Wei, Z., Xu, H., et~al.
\newblock Bernnet: Learning arbitrary graph spectral filters via bernstein
  approximation.
\newblock \emph{Advances in Neural Information Processing Systems}, 34, 2021.

\bibitem[He et~al.(2022)He, Wei, and Wen]{he2022convolutional}
He, M., Wei, Z., and Wen, J.-R.
\newblock Convolutional neural networks on graphs with chebyshev approximation,
  revisited.
\newblock \emph{arXiv preprint arXiv:2202.03580}, 2022.

\bibitem[He et~al.(2020)He, Deng, Wang, Li, Zhang, and Wang]{he2020lightgcn}
He, X., Deng, K., Wang, X., Li, Y., Zhang, Y., and Wang, M.
\newblock Lightgcn: Simplifying and powering graph convolution network for
  recommendation.
\newblock In \emph{Proceedings of the 43rd International ACM SIGIR conference
  on research and development in Information Retrieval}, pp.\  639--648, 2020.

\bibitem[Hoffman \& Kunze(2004)Hoffman and Kunze]{linearalgebra}
Hoffman, K. and Kunze, R.~A.
\newblock \emph{{Linear Algebra}}.
\newblock PHI Learning, second edition, 2004.
\newblock ISBN 8120302702.

\bibitem[Huang et~al.(2020{\natexlab{a}})Huang, He, Singh, Lim, and
  Benson]{huang2020combining}
Huang, Q., He, H., Singh, A., Lim, S.-N., and Benson, A.~R.
\newblock Combining label propagation and simple models out-performs graph
  neural networks.
\newblock \emph{arXiv preprint arXiv:2010.13993}, 2020{\natexlab{a}}.

\bibitem[Huang et~al.(2020{\natexlab{b}})Huang, Rong, Xu, Sun, and
  Huang]{huang2020tackling}
Huang, W., Rong, Y., Xu, T., Sun, F., and Huang, J.
\newblock Tackling over-smoothing for general graph convolutional networks.
\newblock \emph{arXiv preprint arXiv:2008.09864}, 2020{\natexlab{b}}.

\bibitem[Khoshraftar \& An(2022)Khoshraftar and An]{khoshraftar2022survey}
Khoshraftar, S. and An, A.
\newblock A survey on graph representation learning methods.
\newblock \emph{arXiv preprint arXiv:2204.01855}, 2022.

\bibitem[Kipf \& Welling(2017)Kipf and Welling]{kipf2016gcn}
Kipf, T.~N. and Welling, M.
\newblock Semi-supervised classification with graph convolutional networks.
\newblock In \emph{5th International Conference on Learning Representations,
  {ICLR} 2017, Toulon, France, April 24-26, 2017, Conference Track
  Proceedings}, 2017.

\bibitem[Klicpera et~al.(2019)Klicpera, Bojchevski, and
  G{\"{u}}nnemann]{klicpera2019appnp}
Klicpera, J., Bojchevski, A., and G{\"{u}}nnemann, S.
\newblock Predict then propagate: Graph neural networks meet personalized
  pagerank.
\newblock In \emph{7th International Conference on Learning Representations,
  {ICLR} 2019, New Orleans, LA, USA, May 6-9, 2019}, 2019.

\bibitem[Levie et~al.(2018)Levie, Monti, Bresson, and
  Bronstein]{levie2018cayleynets}
Levie, R., Monti, F., Bresson, X., and Bronstein, M.~M.
\newblock Cayleynets: Graph convolutional neural networks with complex rational
  spectral filters.
\newblock \emph{IEEE Transactions on Signal Processing}, 67\penalty0
  (1):\penalty0 97--109, 2018.

\bibitem[Li et~al.(2019)Li, Muller, Thabet, and Ghanem]{li2019deepgcns}
Li, G., Muller, M., Thabet, A., and Ghanem, B.
\newblock Deepgcns: Can gcns go as deep as cnns?
\newblock In \emph{Proceedings of the IEEE/CVF international conference on
  computer vision}, pp.\  9267--9276, 2019.

\bibitem[Li et~al.(2018)Li, Han, and Wu]{li2018deeper}
Li, Q., Han, Z., and Wu, X.
\newblock Deeper insights into graph convolutional networks for semi-supervised
  learning.
\newblock In \emph{Proceedings of the Thirty-Second {AAAI} Conference on
  Artificial Intelligence}, pp.\  3538--3545. {AAAI} Press, 2018.

\bibitem[Lim et~al.(2021)Lim, Hohne, Li, Huang, Gupta, Bhalerao, and
  Lim]{lim2021large}
Lim, D., Hohne, F., Li, X., Huang, S.~L., Gupta, V., Bhalerao, O., and Lim,
  S.~N.
\newblock Large scale learning on non-homophilous graphs: New benchmarks and
  strong simple methods.
\newblock \emph{Advances in Neural Information Processing Systems},
  34:\penalty0 20887--20902, 2021.

\bibitem[Ma et~al.(2021)Ma, Liu, Zhao, Liu, Tang, and Shah]{unifiedgnn_ma}
Ma, Y., Liu, X., Zhao, T., Liu, Y., Tang, J., and Shah, N.
\newblock A unified view on graph neural networks as graph signal denoising.
\newblock In \emph{Proceedings of the 30th ACM International Conference on
  Information \& Knowledge Management}, pp.\  1202--1211, 2021.

\bibitem[Maddox et~al.(2021)Maddox, Tang, Moreno, Wilson, and
  Damianou]{wesley2021nonlinear}
Maddox, W.~J., Tang, S., Moreno, P.~G., Wilson, A.~G., and Damianou, A.~C.
\newblock Fast adaptation with linearized neural networks.
\newblock In \emph{The 24th International Conference on Artificial Intelligence
  and Statistics, {AISTATS}}, volume 130, pp.\  2737--2745. {PMLR}, 2021.

\bibitem[Maurya et~al.(2021)Maurya, Liu, and Murata]{maurya2021simplifying}
Maurya, S.~K., Liu, X., and Murata, T.
\newblock Simplifying approach to node classification in graph neural networks.
\newblock \emph{arXiv preprint arXiv:2111.06748}, 2021.

\bibitem[Pei et~al.(2020)Pei, Wei, Chang, Lei, and Yang]{pei2020geom}
Pei, H., Wei, B., Chang, K. C.-C., Lei, Y., and Yang, B.
\newblock Geom-gcn: Geometric graph convolutional networks.
\newblock \emph{arXiv preprint arXiv:2002.05287}, 2020.

\bibitem[Rong et~al.(2020)Rong, Huang, Xu, and Huang]{rong2020dropedge}
Rong, Y., Huang, W., Xu, T., and Huang, J.
\newblock Dropedge: Towards deep graph convolutional networks on node
  classification.
\newblock In \emph{8th International Conference on Learning Representations,
  {ICLR} 2020, Addis Ababa, Ethiopia, April 26-30, 2020}, 2020.

\bibitem[Rozemberczki et~al.(2021)Rozemberczki, Allen, and
  Sarkar]{rozemberczki2021multi}
Rozemberczki, B., Allen, C., and Sarkar, R.
\newblock Multi-scale attributed node embedding.
\newblock \emph{Journal of Complex Networks}, 9\penalty0 (2):\penalty0 cnab014,
  2021.

\bibitem[Shchur et~al.(2018)Shchur, Mumme, Bojchevski, and
  G{\"u}nnemann]{shchur2018pitfalls}
Shchur, O., Mumme, M., Bojchevski, A., and G{\"u}nnemann, S.
\newblock Pitfalls of graph neural network evaluation.
\newblock \emph{arXiv preprint arXiv:1811.05868}, 2018.

\bibitem[Silva(2020)]{silva2020svdregression}
Silva, T.~S.
\newblock Understanding linear regression using the singular value
  decomposition.
\newblock \emph{https://sthalles.github.io}, 2020.
\newblock URL \url{https://sthalles.github.io/svd-for-regression/}.

\bibitem[Sun et~al.(2022)Sun, Zhang, Zhao, and Yang]{sun2022clar}
Sun, J., Zhang, L., Zhao, S., and Yang, Y.
\newblock Improving your graph neural networks: A high-frequency booster.
\newblock In \emph{IEEE International Conference on Data Mining Workshops
  (ICDMW)}, 2022.

\bibitem[Thanou et~al.(2014)Thanou, Shuman, and Frossard]{thanou2014graphdict}
Thanou, D., Shuman, D.~I., and Frossard, P.
\newblock Learning parametric dictionaries for signals on graphs.
\newblock \emph{{IEEE} Trans. Signal Process.}, 62\penalty0 (15):\penalty0
  3849--3862, 2014.
\newblock \doi{10.1109/TSP.2014.2332441}.
\newblock URL \url{https://doi.org/10.1109/TSP.2014.2332441}.

\bibitem[Velickovic et~al.(2018)Velickovic, Cucurull, Casanova, Romero,
  Li{\`{o}}, and Bengio]{velickovic2018gat}
Velickovic, P., Cucurull, G., Casanova, A., Romero, A., Li{\`{o}}, P., and
  Bengio, Y.
\newblock Graph attention networks.
\newblock In \emph{6th International Conference on Learning Representations,
  {ICLR} 2018, Vancouver, BC, Canada, April 30 - May 3, 2018, Conference Track
  Proceedings}, 2018.

\bibitem[Wang \& Zhang(2022)Wang and Zhang]{wang2022powerful}
Wang, X. and Zhang, M.
\newblock How powerful are spectral graph neural networks.
\newblock In \emph{International Conference on Machine Learning}, 2022.

\bibitem[Wu et~al.(2019{\natexlab{a}})Wu, Souza, Zhang, Fifty, Yu, and
  Weinberger]{wu2019sgc}
Wu, F., Souza, A., Zhang, T., Fifty, C., Yu, T., and Weinberger, K.
\newblock Simplifying graph convolutional networks.
\newblock In \emph{International conference on machine learning}, pp.\
  6861--6871. PMLR, 2019{\natexlab{a}}.

\bibitem[Wu et~al.(2019{\natexlab{b}})Wu, Souza, Zhang, Fifty, Yu, and
  Weinberger]{wu2019simplifying}
Wu, F., Souza, A., Zhang, T., Fifty, C., Yu, T., and Weinberger, K.
\newblock Simplifying graph convolutional networks.
\newblock In \emph{International conference on machine learning}, pp.\
  6861--6871. PMLR, 2019{\natexlab{b}}.

\bibitem[Wu et~al.(2020)Wu, Pan, Chen, Long, Zhang, and
  Philip]{wu2020comprehensive}
Wu, Z., Pan, S., Chen, F., Long, G., Zhang, C., and Philip, S.~Y.
\newblock A comprehensive survey on graph neural networks.
\newblock \emph{IEEE transactions on neural networks and learning systems},
  32\penalty0 (1):\penalty0 4--24, 2020.

\bibitem[Xu et~al.(2018)Xu, Hu, Leskovec, and Jegelka]{xu2018powerful}
Xu, K., Hu, W., Leskovec, J., and Jegelka, S.
\newblock How powerful are graph neural networks?
\newblock In \emph{International Conference on Learning Representations}, 2018.

\bibitem[Yang et~al.(2016)Yang, Cohen, and Salakhudinov]{yang2016revisiting}
Yang, Z., Cohen, W., and Salakhudinov, R.
\newblock Revisiting semi-supervised learning with graph embeddings.
\newblock In \emph{International conference on machine learning}, pp.\  40--48.
  PMLR, 2016.

\bibitem[Zhu \& Koniusz(2020)Zhu and Koniusz]{zhu2020simple}
Zhu, H. and Koniusz, P.
\newblock Simple spectral graph convolution.
\newblock In \emph{International Conference on Learning Representations}, 2020.

\bibitem[Zhu et~al.(2021)Zhu, Wang, Shi, Ji, and Cui]{zhu2021interpreting}
Zhu, M., Wang, X., Shi, C., Ji, H., and Cui, P.
\newblock Interpreting and unifying graph neural networks with an optimization
  framework.
\newblock In \emph{Proceedings of the Web Conference 2021}, pp.\  1215--1226,
  2021.

\end{thebibliography}
\bibliographystyle{icml2023}

\newpage
\appendix
\onecolumn
\section{Derivations and Proofs}
\subsection{Derivation of equation~\eqref{eq:spatial_dl_2}}
\label{app:derive_equ6}
Iterate \eqref{eq:gin} from $H^{(0)}=X$, we have
\begin{align}
    H^{(0)} &= X \\
    H^{(1)} &= \alpha^{(0)}XW_0^{(0)} + \hat{A}XW_1^{(0)} \\ 
    H^{(2)} &= \alpha^{(1)}XW_0^{(1)} + \hat{A}\alpha^{(0)}XW_0^{(0)}W_1^{(1)} + \hat{A}^2XW_1^{(0)}W_1^{(1)} \\ 
    H^{(3)} &= \alpha^{(2)}XW_0^{(2)} + \hat{A}\alpha^{(1)}XW_0^{(1)}W_1^{(2)} \\
            &+ \hat{A}^2\alpha^{(0)}XW_0^{(0)}W_1^{(1)}W_1^{(2)} + \hat{A}^3XW_1^{(0)}W_1^{(1)}W_1^{(2)} \\
    \cdots
\end{align}
Identify the rule of the iteration, we obtain 
\begin{align}
    H^{(k)}=\sum_{i=0}^{l-1}\delta_i^{(k)} + \hat{A}^l X \prod_{h=0}^{l-1}W_1^{(h)},
    \label{eq:apd_spatial_dl_2}
\end{align}
where $\delta_i^{(k)}$ \textcolor{black}{is} calculate by:
\begin{align}
    \delta_i^{(k)} = \alpha^{(k-1-i)} \hat{A}^{i} X W_0^{(k-1-i)} \prod_{j=l-i}^{l-1}W_1^{(j)}.
    \label{eq:apd_spatial_dl_2_delta}
\end{align}
We apply \eqref{eq:apd_spatial_dl_2_delta} on \eqref{eq:apd_spatial_dl_2} and put $\alpha^{(k-1-i)}$ back to the learnable parameters $W_0^{(k-1-i)}$, and we thus have \eqref{eq:spatial_dl_2}.

\subsection{Decomposition of BernNet}
\label{app:deriv_bern}
\textbf{BernNet.}\quad
Different from GPRGNN that utilizes Monomial \textcolor{black}{polynomial}, each term of Bernstein polynomial contains another polynomial, such that:
\begin{align}
    P_k(\hat{L}) &:= \frac{1}{2^K}  \tbinom{K}{k}  (2I-\hat{L})^{K-k} \hat{L}^k \\
    &= \frac{1}{2^K}  \tbinom{K}{k} \sum_{i=0}^{K-k} 2^i (-1)^{K-k-i} \tbinom{K-k}{i} \hat{L}^{K-i} 
\end{align}
This formulation shows that each element of the BernNet, $P_k$, contains a $k$ to $K$-ordered sub-polynomial of $\hat{L}$, where $K$ is the order of a given BernNet. 
Then we merge the same-ordered element in each $P_k$, resulting in feature subspaces $\Phi_t$, for which each term contains the components from $P_0$ to $P_t$. 
\begin{align}
    \Phi_t &= \sum_{j=0}^{t} \frac{1}{2^K}  \tbinom{K}{j} 2^{K-j} \hat{L}^{t} \\
    &= \sum_{j=0}^{t} \frac{1}{2^j}  \tbinom{K}{j}  \hat{L}^{t}
\end{align}

And the corresponding parameter matrix is $\Theta_t = \sum_{j=0}^{t} \gamma^{(j)} W_1W_2$.

\textbf{ChebyNet.}\quad 
Its derivation is almost the same as that for BernNet in feature space. Due to the complex structure of the ChebShev polynomials, we omit this calculation and present a substitute to represent each term in Table~\ref{tab:gnn_elements_big}, noted in the footnote.
Nevertheless, the rules for the construction of feature spaces hold, namely theorem ~\ref{theo:1.1}, and only different forms of weight distribution are applied to current GNNs in the context of this paper.

\subsection{Proofs for Proposition~\ref{theo:1.1}}
\label{app:theo1.1}
First we define the linear correlation of two matrices $M_1,M_2 \in \mathbb{R}^{n\times d}$.
\begin{definition}
If there exists a weight matrix $W \in \mathbb{R}^{d\times d}$ such that $||M_1W - M_2||_2 \rightarrow 0$, we consider $M_2$ to be linearly correlated with $M_1$.
\end{definition}
Then we consider this to be a linear regression problem, i.e. $M_1 W = M_2$. In this expression, each column of $W$ independently returns each column of $M_2$. Without loss of generality, we take an arbitrary column as an example to give the proof.

Suppose $x\in\mathbb{R}^{n\times 1}$ is an arbitrary column of $W$ and $b\in\mathbb{R}^{n\times 1}$ is the corresponding column of $M_2$ to recover: $M_1 x = b$, forming an overdetermined linear system. 
It has no exact solution for a perfect recovery if no assumption is made about $b$, e.g. $b\in \mathrm{Span}(B)$. 
However, its minimum error can be minimized if more entries of $b \rightarrow 0$, since it must have the solution of $\mathbf{0}$ for the corresponding part of $x$.

Then we transfer this case to the two subspaces, i.e. $\Phi_t=U\Lambda^tU^TX$ and $\Phi_{i+t}=U\Lambda^{t+i}U^TX$, and we consider the linear regression problem:
\begin{align}
    \Phi_t W &= \Phi_{t+i} \\
    U\Lambda^tU^TX W &= U\Lambda^{t+i}U^TX\\
    U^TU\Lambda^tU^TX W &= U^TU\Lambda^{t+i}U^TX\\
    \Lambda^tU^TX W &= \Lambda^{t+i}U^TX,
\end{align}
where $\Lambda_{ii}\in[-1,1]$.
In the limiting condition, as $i$ increases, more elements of $\Lambda_{t+i}$ approach $0$, then the corresponding part of the regression problem will be: $(\Lambda^tU^TX)_{\mathcal{C},\cdot} W = 0$ as $W\rightarrow 0$ is required as l2-norm regularization for the rest of the linear regression.
Then, as $i$ increases, more equation constraints can be relaxed as a regularization, leaving a less overdetermined part of the linear system.
Moreover, this can be further relaxed by increasing $t$, because it directly removes more equations from the system. In other words, the linear correlation $W$ is more likely to be obtained. End of proof.

\subsection{Proof for Theorem~\ref{theo:1.2}}
\label{app:theo1.2}


This is quite straightforward. It uses some variations of linear algebra.

Given $\gamma_a\Phi_a W_{B}+ \gamma_b \Phi_b W_{B}= B$, and the linearly correlation $\Phi_a W_a = \Phi_b$, 
then we have,
\begin{align}
    B &= \gamma_a\Phi_a W_{B}+ \gamma_b \Phi_b W_{B} \\
    &= (\gamma_a\Phi_a + \gamma_b \Phi_b) W_B \\
    &= (\gamma_a\Phi_a + \gamma_b \Phi_a W_a) W_B \\
    &= \Phi_a (\gamma_a I + \gamma_b W_a) W_B.   
\end{align}

Therefore, $W_B'=\Phi_a (\gamma_a I + \gamma_b W_a) W_B$. End of the proof.

\subsection{Proof for Theorem~\ref{theo:2.1}}
\label{app:theo2.1}


Given $\Phi_a W_{B}= B$, and two linearly correlated spaces $\Phi_a,\Phi_b\in\mathbb{R}^{n\times d}, \Phi_a W_a = \Phi_b$.

First, in the case of weight sharing, we solve the following linear system with a parameter matrix $W_B'\in\mathbb{R}^{b \times c}$.
\begin{align}
    \gamma_a \Phi_a W_B' + \gamma_b \Phi_b W_B' &= B \\
    \gamma_a \Phi_a W_B' + \gamma_b \Phi_a W_a W_B' &= B \\
    \Phi_a (\gamma_a I + \gamma_b W_a) W_B' &= \Phi_a W_{B},
\end{align}
Without loss of generality, suppose $B$ is uniquely represented by $\Phi_a$ using $W_B$, then the solution is 
$(\gamma_a I + \gamma_b W_a) W_B' = W_{B}$.
Remember that $W_a$ is given and that it is a nearly determined system.

Second, we consider an independent reweighting method, which is the same as flattening feature subspaces. In this case, we solve the following linear system with two parameter matrices $W_B^a, W_B^b\in\mathbb{R}^{b \times c}$. 
\begin{align}
    \Phi_a W_B^a + \Phi_b W_B^b &= B \\
    \Phi_a W_B^a + \Phi_a W_a W_B^b &= B \\
    \Phi_a (W_B^a + W_a W_B^b) &= \Phi_a W_{B} \\ 
   I W_B^a + W_a W_B^b &= W_{B}
\end{align}
This is an underdetermined system, such that $(I, W_a)({W_B^a}^T, {W_B^b}^T)^T=W_B$. 
Therefore, this independent reweighting is much easier to have an optimal solution compared to weight-sharing methods. End of proof.

\subsection{Proof for Theorem~\ref{theo:2.2}}
\label{app:theo2.2}


We provide the proof in two ways.
First, we establish a bound on the minima of the linear repression of the SVD of the regressor matrix. 
Part of the proof~\ref{app:theo1.1}:

Assume a linear regression problem $Mx=b$, and $M$ is the regressor matrix with its SVD $M=U_mS_mV_m^T$.
\begin{align}
    M x &= b \\
    U_m S_m V_m^T x &= b\\
    (U_m S_m V_m^T)^{-1}U_m S_m V_m^T x &= (U_m S_m V_m^T)^{-1} b \leftarrow \mathrm{pseudo-inverse\ using\ \ SVD}\label{eq:inverse_svd2}\\\
    U_m^TU_m x &= V_mS_m^{-1}U_m^T b \label{eq:svd2}
\end{align}
In equation~\eqref{eq:inverse_svd2} we inverse the SVD~\cite{silva2020svdregression} using the invertible property of $U_m S_m V_m^T$, $U_m^TU_m = V_mV_m^T = I$. From the equation~\eqref{eq:svd2} we can see that the best approximation of $b$ is $\hat{b} = U_m U_m^T b$ if $\hat{x} = V_mS_m^{-1}U_m^T b$. 
Therefore, if $U_m U_m^T \rightarrow I$, then perfect recovery can be further approached. 

Then we analyze the change of the regressor matrix between a thin shape feature subspace $\Phi_k$ and a concatenation of $(\Phi_k, U_zS_z)$, where $U_z,S_z$ are the truncated singular vectors and corresponding values of structural matrices, e.g. $\hat{A}$.
\begin{itemize}
    \item[1] For the case where only $\Phi_k$ is used: Since $d << n$, the singular vectors of $\Phi_k$, $U_k$ is column-wise full-rank, resulting in its $n$ rows being highly correlated.
    Given this, $U_kU_k^T$ constructs a rather dense matrix, far from an identity matrix, which leads to a huge gap that uses this feature space as a regressor matrix from a perfect recovery.

    \item[2] In the concatenation case, when $(\Phi_k,U_zS_z)$ is used:
    the singular vector of the concatenation has a permutation difference from the concatenation of their original singular vectors, e.g, 
    \begin{align}
        (\Phi_k,U_zS_z) = (U_kS_kV_k^T, U_zS_z) = (U_k,U_z) \mathrm{DiagCat}(S_k,S_z) \mathrm{DiagCat}(V_k^T,V_z^T) = U'S'V'^T,
    \end{align}
    where $\mathrm{DiagCat}()$ is the diagonal concatenation of two square matrices.
    Since $U'^TU'=I, V'^TV'=I$ and $V'V'^T=I$, this is a reasonable singular vector, but it is a column-wise permutation of $U'$ given the order of $S'$. Therefore, we represent the singular vectors and values of $(\Phi_k,U_zS_z)$ as $(U_k,U_z)P$ and $P^T(S_k,S_z)$, where $P\in\mathbb{R}^{(d+z)\times(d+z) }$ is a unitary permutation matrix. Next, we analyze the unitary property of the singular vector $(U_k,U_z)P$, i.e., if $(U_k,U_z)P ((U_k^T,U_z^T)P)^T = (U_k,U_z) (U_k^T,U_z^T)^T \rightarrow I$.
    
    Since $d << n$, we ignore the influence of $U_k U_z^T$, $U_z U_k^T$ and $U_kU_k^T$ on the result of $(U_k,U_z) (U_k^T,U_z^T)^T$, and the $U_z$ part dominates the unitary of the singular vectors. 
    Given $||U_zS_z-\hat{L}||_2 < \epsilon$, and $\epsilon$ is a small enough constant, a nearly perfect recovery of $\hat{A}$ is achieved by $U_zS_z$. Then the $U_z U_z^T$ is likely to be identity, given the sparsity property of the structural matrix.
\end{itemize}
Therefore, the second concatenation of $(\Phi_k, U_zS_z)$ can achieve a minor error in the regression.
End of proof.

Note that we make no assumptions about the distribution of the node attributes or labels.

\subsection{More discussion of parameter matrices will be stacked}
\label{app:stack}
With a bit of notation abuse, here $W$ is the parameter matrices as $\Theta$ in the equation~\eqref{eq:linear}.
\begin{proposition}
GNNs suppress the parameter space of $W$, leading to a partial expression of all columns of the entire feature space and undermining the benefit of adding redundancy to the original one.
\end{proposition}

\begin{proof}

We summarize the constraints on $W$ in current GNNs as the following:
i) in the case of MLP-based implementation~\cite{he2020lightgcn}, all layers share the same $W$, which forces the layer-wise representation parameters into a single matrix; 
and ii) in the case of layer-wise $W$~\cite{kipf2016gcn}~\cite{li2019deepgcns}, each $W_{k+1}$ is built upon its previous one, i.e,  $W_{k+1} = \prod_{i=0}^{k+1} W_{i}$.
We extract the ideas of these constraints into the following example. Suppose a redundant feature space $U' = (d_0, d_1, \lambda_0d_0, \lambda_1d_1)$, where $d_0 \perp d_1, d_k\in\mathbb{R}^2$. $x \in \mathrm{Span}\{d_0,d_1\}$ need to be recovered by the elements in $U'$.

We deploy the aforementioned two types of constraint on the undecided variables $b_0, b_1, b_2$, and $b_3$:
i) $b_2 = b_0, b_3 = b_1$, and ii) $b_2 = \mu b_0, b_3 = \mu b_1$, where $\mu$ is a trainable scalar. 
They align with the graph neural networks.
We begin by discussing these two cases. 

Representing $x$ in the first case, yields:   
\begin{align}
    x &= b_0d_0 + b_1d_1 + b_0 \lambda_0 d_0 + b_1 \lambda_1 b_1\\
    &= (1+\lambda_0)b_0d_0 + (1+\lambda_1)b_1d_1.
\end{align}
Using the unique representation theorem~\cite{linearalgebra}, we have $(1+\lambda_0)b_0 = a_0$ and $(1+\lambda_1)b_1 = a_1$. Put it in a matrix multiplication format: 
\begin{align}
    \begin{pmatrix}
    1 & 0 & \lambda_0 & 0\\
    0 & 1 & 0 & \lambda_1\\
    \end{pmatrix}
    \begin{pmatrix}
    b_0 \\ b_1 \\ b_0 \\ b_1\\
    \end{pmatrix}
    =
    \begin{pmatrix}
    a_0 \\ a_1
    \end{pmatrix},
\end{align}
which produces:
\begin{align}
    \begin{pmatrix}
    1 + \lambda_0 & 0\\
    0 & 1 + \lambda_1\\
    \end{pmatrix}
    \begin{pmatrix}
    b_0 \\ b_1 \\
    \end{pmatrix}
    =
    \begin{pmatrix}
    a_0 \\ a_1
    \end{pmatrix}.
\end{align}
It holds the closed form that $b_0 = \frac{a_0}{(1+\lambda_0)}, b_1 = \frac{a_1}{(1+\lambda_1)}$.

Then, we represent $x$ in the second case:
\begin{align}
    x &= b_0d_0 + b_1d_1 + \mu b_0 \lambda_0 d_0 + \mu b_1 \lambda_1 b_1\\
    &= (1+\mu\lambda_0)b_0d_0 + (1+\mu\lambda_1)b_1d_1,
\end{align}
which produces $(1+\mu\lambda_0)b_0 = a_0$ and $(1+\mu\lambda_1)b_1 = a_1$. Formulate them in a matrix multiplication:
\begin{align}
    \begin{pmatrix}
    1 & 0 & \lambda_0 & 0\\
    0 & 1 & 0 & \lambda_1\\
    \end{pmatrix}
    \begin{pmatrix}
    b_0 \\ b_1 \\ \mu b_0 \\ \mu b_1\\
    \end{pmatrix}
    =
    \begin{pmatrix}
    a_0 \\ a_1
    \end{pmatrix}.
\end{align}
This is a under-determined system and gives $b_0 = \frac{a_0}{(1+\mu\lambda_0)}$, $b_1 = \frac{a_1}{(1+\mu\lambda_1)}$, $b_2 = \frac{\mu a_0}{(1+\mu\lambda_0)}$, and $b_3 = \frac{\mu a_1}{(1+\mu\lambda_1)}$.

We look into the values of $b_k$ to get the expressivity of the base $D$. 
Given the extreme case where $\lambda_0 \rightarrow 0$, the appended $\lambda_0d_0$ is constrained while the original one keeps expressing. On the contrary, when $\lambda_0 \rightarrow \infty$, the original base $d_0$ is constrained by $\frac{a_0}{(1+\lambda_0)}$ or $\frac{a_0}{(1+\mu\gamma_0)}$ while the appended one expresses. 
Besides, for the second case, when $\mu \rightarrow 0$, the corresponding bases are limited by $b_2, b_3 \rightarrow 0$. Consequently, both cases lead to partial expression of the the whole feature space.

Finally, compared these two cases, i.e., (60) and (63) to (56), we find that they merely restrict the parameter space of $(b_0, b_1, b_2, b_3)^T$ by either sharing the values of each other or enforcing their linear dependence.
Therefore, restricting the parameter space in these two cases leads to partial expression of the the whole feature space. 
This proof is completed.
\end{proof}

\subsection{The column-wise normalization in FE-GNN and in current GNNs}
\label{app:atom}
\textbf{Does the column-wise normalization matter? It matters if the node scale is not huge.}\quad 
As shown in Table~\ref{tab:fegnn_ablation}, we find that column-wise normalization works well in most cases, except for PubMed.
This may be because the large node scale of PubMed causes the value of the normalized feature space to be tiny. 

Then, we include some 10-order polynomial functions to see the different responses of column-wise normalization of these models. 
Column-wise normalization is defined as forcing $\lVert F_{\cdot i}\rVert_2=1$, where we take $F$ as the concatenation of the feature space. 
We extend this to an arbitrary $k$ times $\lVert F_{\cdot i}\rVert_2=1$, i.e., $\lVert F_{\cdot i}\rVert_2=k$, 
which is equivalent to measuring the degree of consistency of each $\lVert F_{\cdot i}\rVert_2$.
Therefore, we report the standard variance of $\{\lVert F_{\cdot i}\rVert_2;i=1,2,\cdots\}$, and the smaller the value, the greater the response of the column-wise normalization. 

\begin{table}[h]
    \centering
    \caption{The column-wise normalization response for different polyonmials on Cora}
    \begin{tabular}{c|ccc}
    \toprule
         & ChebyShev polynomial & Bernstein polynomial & Monomial polynomial\\
    \midrule
        Cora & 3.8246 & 4.7044e-06 & 0.4947 \\
        CiteSeer & 34.6432 & 0.0023 & 24.4430 \\
        Chameleon & 660.4274 & 0.7308 & 1039.2469 \\
        Squirrel & 245.6538 & 0.7063 & 700.9365 \\
    \bottomrule
    \end{tabular}
    \label{tab:atom}
\end{table}

Chebshev, Bernstein, and Monomial polynomials are compared in Table~\ref{tab:atom}.  
The Bernstein polynomial produces the least variance, indicating that it
it promotes the most atomicity compared to other polynomials. 
This observation is consistent with the narrative in the original BernNet paper~\cite{he2021bernnet}, where the authors claim that the Bernstein polynomial is more numerically stable than other polynomial functions.

\section{Experimental settings}
\label{app:settings}
\subsection{Dataset details}
The datasets are concluded in Table~\ref{tab:dataset}, with licenses. 
\footnote{Chameleon, Squirrel: https://github.com/benedekrozemberczki/MUSAE/blob/master/LICENSE}
\footnote{Cora, CiteSeer, PubMed, Actor: https://networkrepository.com/policy.php}
\footnote{Computers, Photo: https://github.com/shchur/gnn-benchmark/blob/master/LICENSE}
Cora, CiteSeer, and PubMed are commonly used homophilic citation networks~\cite{yang2016revisiting}.
Computers and Photo are homophilic co-bought networks from Amazon~\cite{shchur2018pitfalls}.
For heterophilic datasets, we utilize hyperlinked networks Squirrel and Chameleon from \cite{pei2020geom}, 
and Actor, a subgraph from the film-director-actor network~\cite{rozemberczki2021multi}.
PyG\footnote{https://pytorch-geometric.readthedocs.io/en/latest/modules/datasets.html} are employed to get these data.
Each datasets are split into three parts using random selection:  60\% as the training set, 20\% as the validation set, and 20\% as the test set.
We set these datasets to undirected graphs as we assumed in the Preliminaries.

\begin{table*}[h]
    \centering
    \footnotesize
    \caption{Statistics of Datasets}
    \begin{tabular}{cccccc|ccc}
    \toprule
    \footnotesize
    & Cora & CiteSeer & PubMed & Computers & Photo & Squirrel & Chameleon & Actor \\
    \midrule
    $|\mathcal{V}|$ & 2,708 & 3,327 & 19,717 & 13,752 & 7,650 & 5,201 & 2,277 & 7,600\\
    $|\mathcal{E}|$ & 5,278 & 4,552 & 44,338 & 245,861 & 119,081 & 217,073 & 36,101 & 30,019 \\
    \# Features & 1433 & 3703 & 500 & 767 & 745 & 128 & 128 & 932 \\
    $h(\mathcal{G})$ & 0.81 & 0.74 & 0.80 & 0.78 & 0.83 & 0.22 & 0.23 & 0.22 \\    
    $d(\mathcal{G})$ & 1.95 & 1.37 & 2.25 & 17.88 & 15.57 & 41.74 & 15.85 & 3.95 \\
    \bottomrule 
    \end{tabular}
    \label{tab:dataset}
\end{table*} 

We report the average accuracy~(micro F1 score) in the classification task with a 95\% confidence interval in all the tables and figures. For each result, we run 100 times on 10 random seeds. 
We employ Adam for optimization and set
the early stopping criteria as a warmup of $50$ pluses patience of $200$ for a maximum of $100$ epochs.
We conduct all the experiments on the machine with NVIDIA 3090 GPU~(24G) and Intel(R) Xeon(R) Platinum 8260L CPU @ 2.30GHz. 

\subsection{Searching space for baselines hyper-parameters}
For FE-GNN, we turn the following hyper-parameters by the grid search.
\begin{itemize}
    \item Learning rate: $\{0.01,0.05,0.1\}$
    \item Weight decay: $\{0.0005,0.001,0.005,0.01,0.02,0.05\}$
    \item $|S|$ for homophilic graphs: $\{0,10,50,100,200,500,1000,2000\}$
    \item $|S|$ for heterophilic graphs: $\{500,600,700,800,900,1000,1500,2000\}$
    \item Suggested $|S|$: the whole hundred from the 94\% singular values
    \item Hidden size: $64$
    \item Ranks $k$ of the polynomial $P_k(\hat{L})$: $\{0,1,2,3\}$
\end{itemize}

\label{appexp:hyper}
\begin{table}[h]
\caption{The universally used hyper-parameters for FE-GNN.}
    \centering
    \begin{tabular}{c|cccccccc}
    \toprule
            & lr & weight decay & $\vert S\vert$ & hidden & $k$\\
        \midrule
        Cora &  0.01 & 0.01 & 50 & 64 & 3 \\
        CiteSeer & 0.01 & 0.02 & 100 & 64 & 1 \\
        PubMed & 0.01 & 0.005 & 100 & 64 & 3 \\
        Computers & 0.01 & 0.0005 & 1000 & 64 & 3 \\
        Photo & 0.01 & 0.0005 & 500 & 64 & 3 \\
        Squirrel & 0.01 & 0.001 & 2000 & 64 & 3 \\
        Chameleon & 0.01 & 0.0005 & 700 & 64 & 3 \\
        Actor & 0.01 & 0.001 & 10 & 64 & 0 \\
        \bottomrule
    \end{tabular}
    \label{tab:hyper_gcf}
\end{table}

\begin{table}[h]
\caption{The turned hyper-parameters for the baselines.}
    \centering
    \begin{adjustbox}{width=\textwidth}
    \begin{tabular}{c|cccccccc}
    \toprule
            & lr & weight decay & dropout & hidden & layers/ranks & others\\
        \midrule
        MLP &  \{0.01, 0.05\} & 0.0005 & \{0.5, 0.6, 0.8\} & 64 & 2 & -\\
        \midrule
        GCN & \{0.01, 0.05\} & 0.0005 & \{0.5, 0.6, 0.8\} & 64 & \{2,3\} & -\\
        GAT & \{0.01, 0.05\} & 0.0005 & \{0.5, 0.6, 0.8\} & 64 & \{2,3\} & $heads$:\{1,8\}\\
        GraphSAGE & \{0.01, 0.05\} & 0.0005 & \{0.5, 0.6, 0.8\} & 64 & \{2,3\} & -\\
        GCNII & \{0.01, 0.05\} & 0.0005 & 0.5 & 64 & \{2,4,10\} &$\alpha,\theta$:\{0.1, 0.2, 0.5, 0.8, 0.9\}\\
        APPNP & \{0.01, 0.05\} & 0.0005 & 0.5 & 64 & \{2,3,4,5,8\} &$\alpha$:\{0.1, 0.2, 0.5, 0.8, 0.9\}\\
        \midrule
        ChebNet & \{0.005, 0.01, 0.05\} & \{0.0, 0.0005\} & \{0.1, 0.2, 0.5\} & 64 & 10 & -\\
        GPRGNN & \{0.005, 0.01, 0.05\} & \{0.0, 0.0005\} & \{0.1, 0.2, 0.5\} & 64 & 10 & -\\
        \multirow{2}{*}{BernNet} & \multirow{2}{*}{\{0.005, 0.01, 0.05\}} & \multirow{2}{*}{\{0.0, 0.0005\}} & \multirow{2}{*}{\{0.1, 0.2, 0.5\}} & \multirow{2}{*}{64} & \multirow{2}{*}{10} & $prop\_drate$:\{0.001,0.02,0.01,0.05\}\\
        &&&&&&$prop\_lr$:\{0.0, 0.1, 0.2, 0.5, 0.6, 0.7, 0.9\}\\
        \midrule
        ADA-GNN & \{0.05, 0.01\} & \{0.0005, 0.00005\} & \{0.2,0.5,0.8\} & 64 & \{2,5,10\} & $s$:\{1,9,19,29\}\\
        GNN-LF & 0.01 & 0.005 & 0.5 & 64 & 10 & $\alpha,\mu$: \{0.1, 0.2, 0.3, 0.4, 0.5, 0.6, 0.8, 0.9\} \\
        GNN-HF & 0.01 & 0.005 & 0.5 & 64 & 10 & $\alpha,\beta$: \{0.1, 0.2, 0.3, 0.4, 0.5, 0.6, 0.8, 0.9\} \\
        \bottomrule
    \end{tabular}
    \end{adjustbox}
    \label{tab:hyper_base}
\end{table}

Table~\ref{tab:hyper_gcf} represents the hyper-parameters searched for the baselines used in our experiments. We prioritize their original released code repository, and the ranges of turning parameters are according to their papers.
\begin{itemize}
    \item MLP, GCN, GAT GraphSAGE, APPNP, GCNII are implemented with PyG.~\footnote{https://github.com/pyg-team/pytorch\_geometric}
    \item ChebNet is implemented according to the code style of BernNet/GPRGNN.
    \item GPRGNN is implemented according to its original code repository.~\footnote{https://github.com/jianhao2016/GPRGNN}
    \item BernNet is implemented according to its original code repository.~\footnote{https://github.com/ivam-he/BernNet} 
    \item ADA-UGNN is implemented according to its original code repository.~\footnote{https://github.com/alge24/ADA-UGNN} 
    \item GNN-HF/LF are implemented according to its original code repository.~\footnote{https://github.com/zhumeiqiBUPT/GNN-LF-HF} 
\end{itemize}

\subsection{Other transformations for compacting graph structure information}
\label{app:other_sj}
We add other possible transformations to extract compacted information from the normalized adjacency matrix $\hat{A}$.
We compare them in detail:
\begin{itemize}
    \item KernelPCA: a non-linear kernel PCA method using the Radial Basis Function~(RBF).
    \item FastICA: a fast version of Independent Component Analysis, which is a linear method.
    \item IsoMap: a nonlinear dimensionality reduction method based on spectral theory.
    \item LINKX~\cite{lim2021large}: an MLP architecture for graph-structured data that includes the graph adjacency matrix as part of the feature space. Due to its inferior performance compared to other baselines, such as GCNII, we exclude it from our main comparisons in the main context.
    
\end{itemize}
All of them can be easily implemented using the $\mathtt{sklearn}$ package. 
As shown in the table~\ref{tab:other_sj}, our chosen truncated SVD has comparable performance and we stick with it for further analysis in the main text.
We regard further investigation of the complex extraction methods as future work, which is beyond the scope of this paper.

\begin{table}[h]
    \centering
    \begin{adjustbox}{width=0.6\columnwidth}
    {\color{black}\begin{tabular}{cccccc}
    \toprule
        & Cora & CiteSeer & Chameleon & Squirrel & Photo \\
        \midrule
        None & \underline{$89.20 _{\pm0.93}$} & $81.95 _{\pm0.87}$ & $61.54 _{\pm1.52}$ & $43.21 _{\pm0.99}$ & -  \\
            \cellcolor{Gray}Truncated-SVD & \cellcolor{Gray}$\textbf{89.45}_{\pm0.22}$ & \cellcolor{Gray}\underline{$81.96 _{\pm0.23}$} & \cellcolor{Gray}\underline{$73.33 _{\pm0.35}$} & \cellcolor{Gray}$67.90 _{\pm0.23}$ & \cellcolor{Gray}{$\textbf{95.45} _{\pm0.15}$}  \\
        KernelPCA & $88.61 _{\pm0.82}$ & {$\textbf{81.99} _{\pm1.11}$} & {$\textbf{73.66} _{\pm1.45}$} & $\textbf{68.79} _{\pm1.13}$ & \underline{$95.36 _{\pm0.51}$}\\
        FastICA & $88.77 _{\pm1.09}$ & $81.92 _{\pm1.00}$ & $73.32 _{\pm1.37}$ & \underline{$68.12 _{\pm0.97}$} & $95.30 _{\pm0.22}$ \\
        IsoMap & $88.54 _{\pm0.86}$ & $82.07 _{\pm1.15}$ & $67.00 _{\pm1.54}$ & $54.47 _{\pm0.87}$ & $94.88 _{\pm0.34}$\\
        LINKX & - & - &  $68.42_{\pm 1.38}$ & $61.81_{\pm 1.80} $ & - \\
        \bottomrule
    \end{tabular}}
    \end{adjustbox}
    \caption{\color{black}Comparing the transformations in compacting the normalized adjacency matrix}
    \label{tab:other_sj}
\end{table}

\subsection{Results of FE-GNN using different polynomial orders}
\label{app:best_k}
In the main text, we implement the polynomial order $K$ in the range of three based on the empirical observations, e.g. Figure~\ref{fig:fegnn_oversmoothing}. Here we provide more comprehensive results for different choices of $K$.

Table~\ref{tab:best_K} shows that we can find better $K$ in a wider range, although the improvement may be marginal.
\begin{table}[h]
    \centering
    {\color{black}
    \begin{adjustbox}{width=0.8\textwidth}
    \begin{tabular}{ccccccc}
    \toprule
       $K$ & 4 & 5 & 6 & 7 & 8 & 9\\
\midrule
Cora & $89.60 \pm 0.30$ & $89.44 \pm 0.25$ & $89.52 \pm 0.26$ & $89.35 \pm 0.24$ & $89.34 \pm 0.22$ & $89.08 \pm 0.25$ \\
CiteSeer & $80.66 \pm 1.09$ & $81.15 \pm 1.06$ & $81.11 \pm 0.89$ & $80.83 \pm 1.07$ & $80.54 \pm 1.03$ & $80.10 \pm 1.02$ \\
Computers & $91.01 \pm 0.46$ & $90.90 \pm 0.51$ & $90.98 \pm 0.42$ & $90.77 \pm 0.39$ & $90.82 \pm 0.41$ & $90.45 \pm 0.32$ \\
Chameleon & $73.42 \pm 0.40$ & $73.69 \pm 0.43$ & $73.62 \pm 0.43$ & $73.68 \pm 0.42$ & $73.80 \pm 0.39$ & $73.75 \pm 0.38$ \\
Squirrel & $68.26 \pm 0.78$ & $68.41 \pm 0.88$ & $68.55 \pm 0.82$ & $68.83 \pm 0.68$ & $68.92 \pm 0.76$ & $69.06 \pm 0.93$ \\
\bottomrule
    \end{tabular}
    \end{adjustbox}}
    \caption{\color{black}Comparison of different $K$ within $10$}
    \label{tab:best_K}
\end{table}


\begin{figure}[h]
    \centering
    \includegraphics[width=\textwidth]{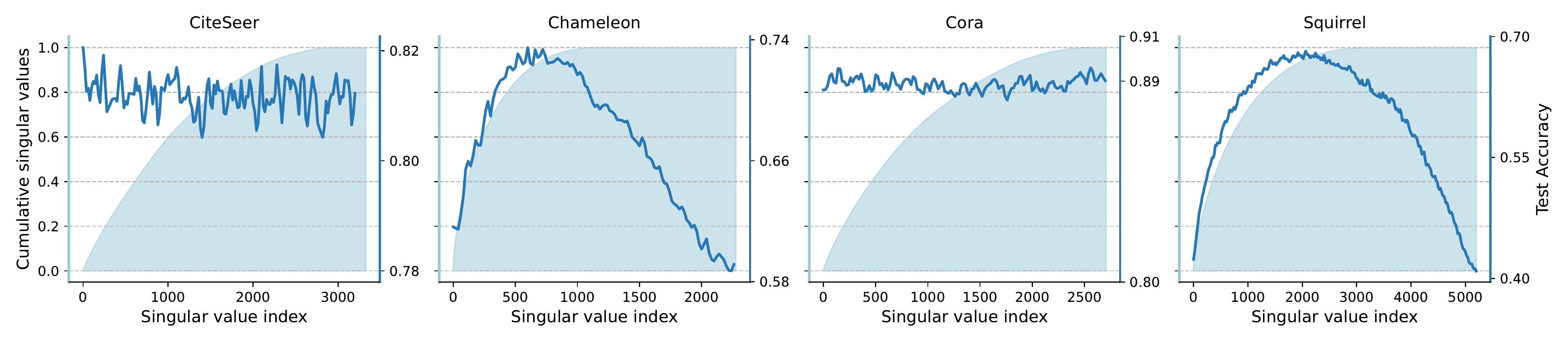}
    \caption{\color{black}More results of SVD}
    \label{fig:app_svd}
\end{figure}

\subsection{\textcolor{black}{Results of FE-GNN on public splits}}
\textcolor{black}{We additionally add experiments of FE-GNN on the public splits from Geom-GCN~\cite{pei2020geom}, as shown in the following Table~\ref{tab:rebuttal_public}. We can see that our proposal still achieves comparable performance. However, the linearity of FE-GNN may constrain its capability in few-shot settings, which is also observed in SGC~\cite{wu2019simplifying}.}

\begin{table}[]
    \centering
    \caption{\textcolor{blue}{Comparisons of FE-GNN on public splits}}
    \begin{tabular}{cccc}
    \toprule
     & Cora & CiteSeer & PubMed \\
\midrule
\cellcolor{Gray} FE-GNN (Ours) & \cellcolor{Gray}$89.13 \pm 1.30 $ & \cellcolor{Gray}$79.13 \pm 1.54$ & \cellcolor{Gray}$89.88 \pm 0.64$ \\
 GCNII & $88.37 \pm 1.25$ & $77.33 \pm 1.48$ & $90.15 \pm 0.43$ \\
    \bottomrule
    \end{tabular}
    \label{tab:rebuttal_public}
\end{table}

\subsection{More about SVD}
\label{app:svd}
In the main test, we append the results to CiteSeer and Squirrel to better verify the importance of the principal components in extracting information from the adjacency matrix in $S_j$.
As shown in Figure~\ref{fig:app_svd} below, we find the exact results we shared in section~\ref{sec:experiments}. Cora and CiteSeer both 1) have a smoothing distribution of the singular values and 2) the information from the graph structure is less important than the node features and their interaction, therefore the change in performance is more stable with the introduction of more principal components. On the other hand, Chameleon and Squirrel 1) have a more centralized distribution of singular values and 2) graph structure is more important information, resulting in a tendency of performance to first increase and then decrease. In general, we can achieve satisfactory results on both types of data sets when 94\% of the principal components are included.

Here, we offer more intuition about the use of principal components. The principal components project and summarize larger correlated variables into smaller and more interpretable axes of variation. It is ideal for $S_j$ to embody the graph structure information from the adjacency matrix because the adjacency matrix is sparse and high-dimensional, but each node is topologically correlated. However, the different components must be distinct from each other to be interpretable, otherwise, they just represent random directions, which leads to noise.

\subsection{\textcolor{black}{Comparisons of structural principal components to a random/trainable bias term}}
\textcolor{black}{To further evaluate the contribution of the proposed SVD space - we would like to emphasize that the SVD space encodes the structure information of the adjacency matrix, which should also be a critical character of GNNs' feature space. To verify that this structure information is indispensable and cannot be replaced by a structure-free matrix, we provide the following two experiments, both of which perform the replacement on the SVD part $SW_s$, where $S$ is the fixed structural principal components matrix and $W_s$ is the corresponding learnable weight matrix.}

\textcolor{black}{The first experiment, labeled '$CW_c$~(random/orthogonal)', replaces the structural components $S$ as a random matrix $C$ with the corresponding weights $W_c$ to be learned. The other setting, denoted '$C$~(random/orthogonal)', replaces the sum $SW_s$ as a randomly initialized bias term $C$ to be learned. For a fair comparison with the orthogonal SVD $S$, we add orthogonal initialization variants to both experiments, denoted '(orthogonal)'. And the '(random)' suffix represents the 'xavier\_normal' initialization. The results are shown in Table~\ref{tab:rebuttal_cwc}; it shows both $C$ and $CW_c$ perform worse than our proposal, especially on the heterophilic Chameleon and Squirrel datasets, which justifies our discussion of the structural principal components.}

\begin{table}[]
    \centering
    \caption{\textcolor{blue}{Comparisons of structural principal components and learnable bias}}
    \begin{tabular}{ccccccc}
    \toprule
& Cora & CiteSeer & Chameleon & Squirrel & Computers & Photo \\
\midrule
\cellcolor{Gray}FE-GNN (ours) & \cellcolor{Gray}$89.45 \pm 0.22$ & \cellcolor{Gray}$81.96 \pm 0.23$ & \cellcolor{Gray}$73.33 \pm 0.35$ & \cellcolor{Gray}$67.90 \pm 0.23$ & \cellcolor{Gray}$90.79 \pm 0.08$ & \cellcolor{Gray}$95.45 \pm 0.15$ \\
FE-GNN w/o $S$ & $89.20 \pm 0.93$ & $81.95 \pm 0.87$ & $61.54 \pm 1.52$ & $43.21 \pm 0.99$ & $88.48 \pm 0.80$ & $94.94 \pm 0.79$ \\
$CW_c$ (random) & $89.07 \pm 1.17$ & $81.54 \pm 1.38$ & $54.93 \pm 1.89$ & $43.13 \pm 1.37$ & $89.54 \pm 0.83$ & $94.56 \pm 1.05$ \\
$CW_c$ (orthogonal) & $89.31 \pm 1.20$ & $81.65 \pm 1.34$ & $55.96 \pm 2.00$ & $36.87 \pm 1.09$ & $88.62 \pm 0.64$ & $94.87 \pm 0.73$ \\
$C$ (random) & $89.03 \pm 1.34$ & $80.83 \pm 1.20$ & $61.03 \pm 2.18$ & $43.32 \pm 1.25$ & $88.85 \pm 0.89$ & $93.46 \pm 3.55$ \\
$C$ (orthogonal) & $89.16 \pm 1.28$ & $80.99 \pm 1.24$ & $61.17 \pm 2.09$ & $43.09 \pm 1.34$ & $88.94 \pm 0.78$ & $94.84 \pm 0.79$ \\
\bottomrule
    \end{tabular}
    \label{tab:rebuttal_cwc}
\end{table}
\subsection{\textcolor{black}{More results of FSGNN~\cite{maurya2021simplifying}}}
\label{app:rebuttal_fsgnn}

\textcolor{black}{We append some experimental results of FSGNN; we adopt the grid search on the range that their original paper reported (Table 5 of \cite{maurya2021simplifying}), and the results are shown in Table~\ref{tab:rebuttal_fsgnn}. We could see that it cannot achieve better performance on both homophilic and heterophilic datasets.}

\begin{table}[]
    \centering
    \caption{\textcolor{blue}{Comparisons of FSGNN on random 60/20/20 splits}}
    \begin{tabular}{cccccc}
    \toprule
& Cora & CiteSeer & PubMed & Chameleon  & Squirrel  \\
\midrule
\cellcolor{Gray}FE-GNN (Ours) & \cellcolor{Gray}$89.45 \pm 0.22$ & \cellcolor{Gray}$81.96 \pm 0.23$ & \cellcolor{Gray}$90.79 \pm 0.08$ & \cellcolor{Gray}$73.33 \pm 0.35$ & \cellcolor{Gray}$67.90 \pm 0.23$ \\
FSGNN & $87.01 \pm 1.61$ & $79.45 \pm 1.78$ & $90.71 \pm 0.68$  & $66.33 \pm 1.04$ & $54.62 \pm 1.57$ \\
\bottomrule
    \end{tabular}
    \label{tab:rebuttal_fsgnn}
\end{table}

\section{\textcolor{black}{Additional Feature Space Perspective Explanations}}

\textcolor{black}{We give a simple illustration from the point of view of using graph structure data, where node attributes and graph structure are the only sources of information. We can think of a pie chart, as shown in Figure~\ref{fig:rebuttal_feature_space}, where the proportions in blue, yellow, and green colors represent node attributes (e.g., $X$), graph structure (e.g., $S$), and their combination (e.g., $LX$)'s contribution to the overall expressiveness of the feature space. Under different conditions of the graph, each part has different proportions.
For example, if the graph is more homophilic, then the combination of the part (in green) contributes more. And if the dimensionality of node attributes is limited, then the proportion of node attributes (in blue) is compact. In general, the task of graph representation learning is to take full advantage of all these three pieces of information. Therefore, from a feature space view, it naturally treats each piece of information equally and aims to construct an extensive feature space that includes all three pieces in the pie chart. In addition, it shows that FE-GNN should be less effective when $d\rightarrow n$, because the feature space is essentially adequate.}

\textcolor{black}{However, traditional GNN models consider the graph structure as an information path to do message-massing, i.e., propose to aggregate more hops from neighbors. Since they are restricted to avoid that the graph structure is actually another information source, they could only include two parts in the piechart where the individual graph structure (in yellow) is missing. Therefore, if the node attributes are limited and the graph obeys heterophily, only the feature space view can do a good job by including the graph structure as an individual information source.}

\textcolor{black}{Furthermore, the difference in feature space between the homophilic and heterophilic graphs could be: the combination part (in green) of the homophilic pie chart is much larger than that of the heterophilic one. This difference is particularly strong when we look at another part of the feature subspace - the node attributes. The proportion of node attributes $X\in\mathbb{R}^{n\times d}$ (in blue) of the pie chart depends on the dimensionality $d$. If $d\rightarrow n$, (we reasonably assume that the high-dimensional node attributes are sparse here), then the node attributes themselves could form an extensive feature space, making the proportion large and the individual graph structure less important. However, if both the graph obeys heterophily and the node attributes are limited, then the remaining feature subspace of the graph structure (in yellow) is so important; this is also consistent with our ablation study in Table 5, especially for Chameleon and Squirrel.}

\begin{figure}[h]
    \centering
    \includegraphics[width=0.9\textwidth, trim={0 6cm 0 3cm}, clip]{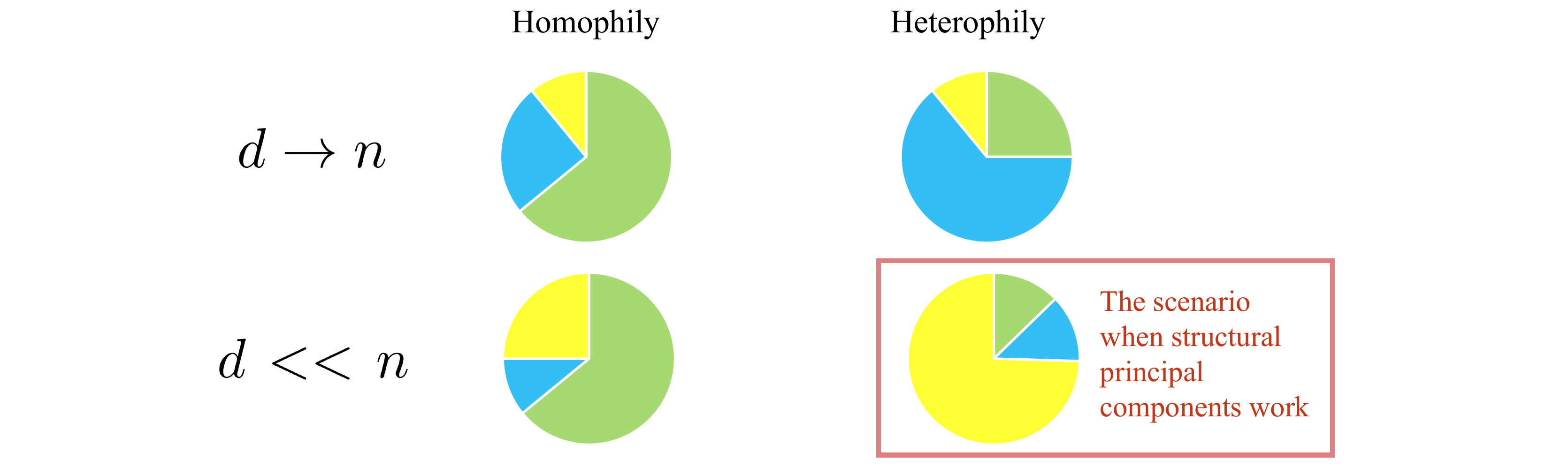}
    \caption{\textcolor{black}{Example diagram of feature space}}
    \label{fig:rebuttal_feature_space}
\end{figure}


\end{document}